\documentclass[letterpaper]{article} 
\usepackage{aaai2026}  
\usepackage{times}  
\usepackage{helvet}  
\usepackage{courier}  
\usepackage[hyphens]{url}  
\usepackage{graphicx} 
\usepackage{natbib}  
\usepackage{caption} 
\usepackage{algorithm}
\usepackage{algorithmic}

\usepackage{newfloat,multicol}
\usepackage{listings}
\DeclareCaptionStyle{ruled}{labelfont=normalfont,labelsep=colon,strut=off} 
\floatstyle{ruled}
\newfloat{listing}{tb}{lst}{}
\floatname{listing}{Listing}
\usepackage{amsmath,amssymb}
\usepackage{amsthm}
\usepackage{booktabs}
\usepackage[utf8]{inputenc}
\usepackage{listings}
\usepackage{soul}
\usepackage{tikz}
\usepackage{xcolor}

\usetikzlibrary{arrows, arrows.meta, calc, positioning, shapes}

\title{Regular Games -- an Automata-Based General Game Playing Language}
\author{
    Radosław Miernik,
    Marek Szykuła
    Jakub Kowalski,\\
    Jakub Cieśluk,
    Łukasz Galas,
    Wojciech Pawlik
}
\affiliations{
    Institute of Computer Science, University of Wrocław, Wrocław, Poland\\
    \{radoslaw.miernik, msz, jko\}@cs.uni.wroc.pl,
    \{jkciesluk, lukszgalas, pawlik.wj\}@gmail.com
}

\newtheorem{theorem}{Theorem}
\newtheorem{definition}{Definition}
\newtheoremstyle{TheoremNum}
        {\topsep}{\topsep}              
        {\itshape}                      
        {}                              
        {\bfseries}                     
        {.}                             
        { }                             
        {\thmname{#1}\thmnote{ \bfseries #3}}
    \theoremstyle{TheoremNum}

\newcommand{\Rust}[0]{\lstinline}

\lstdefinelanguage{RG}{
    morecomment=[l]{//},
    literate=
        {DOLLAR}{$\mbox{\color{teal}\texttt{\$}}$}{1} 
        {UNDERSCORE}{\_}{1}
        {_1}{$\textsubscript{1}$}{1}
        {_2}{$\textsubscript{2}$}{1}
        {_i}{$\textsubscript{i}$}{1}
        {_n}{$\textsubscript{n}$}{1}
        {->}{$\rightarrow$}{1}
        {:=}{$\prec$}{1}
        {...}{$\dots$}{2}
}
\lstdefinestyle{RG}{
    language={RG},
    aboveskip={3pt},
    belowskip={3pt},
    commentstyle=\color{gray!80!},
    keywords={const,type,var},
    keywordstyle=\color{blue},
    keywords=[2]{begin,end,p0,p1,p2,q1,q2,q3,r1,r2, a, b, c, d, a1, a2, a3, target, b0, b1, b2, turn, move, _w, _u, t1, t2, t3, x0, x1, x2, y0, y1},
    keywordstyle=[2]\color{purple},
    mathescape=true,
}
\newcommand{\RG}[0]{\lstinline[style=RG]}

\lstdefinelanguage{HRG}{
    morecomment=[l]{//},
    literate=
        {DOLLAR}{$\mbox{\color{teal}\texttt{\$}}$}{1} 
        {->}{$\rightarrow$}{1}
        {:=}{$\prec$}{1}
        {...}{$\dots$}{2}
}
\lstdefinestyle{HRG}{
    language={HRG},
    aboveskip={3pt},
    belowskip={3pt},
    commentstyle=\color{gray!80!},
    keywords={branch,domain,forall,graph,if,in,loop,reusable,where,while},
    keywordstyle=\color{blue},
    keywords=[2]{check,end,not,reachable,return},
    keywordstyle=[2]\color{purple},
    mathescape=true,
}

\begin{document}

\maketitle

\begin{abstract}
We propose a new General Game Playing (GGP) system called Regular Games (RG).
The main goal of RG is to be both computationally efficient and convenient for game design.
The system consists of several languages.
The core component is a low-level language that defines the rules by a finite automaton.
It is minimal with only a few mechanisms, which makes it easy for automatic processing (by agents, analysis, optimization, etc.).
The language is universal for the class of all finite turn-based games with imperfect information.
Higher-level languages are introduced for game design (by humans or Procedural Content Generation), which are eventually translated to a low-level language.
RG generates faster forward models than the current state of the art, beating other GGP systems (Regular Boardgames, Ludii) in terms of efficiency.
Additionally, RG's ecosystem includes an editor with LSP, automaton visualization, benchmarking tools, and a debugger of game description transformations.
\end{abstract}

\begin{links}
    \link{Code}{https://github.com/radekmie/rg}
    \link{Compiler}{https://github.com/WoojtekP/RGcompiler}
\end{links}

\section{Introduction}
    Generalization is a natural path of development for Artificial Intelligence solutions. 
    Achieving mastery in a small domain leads to transplanting the idea to other problems and eventually bringing it all together to cover a generalized domain using a unified approach. 
    DeepMind's AlphaGo project is an excellent illustration of this process for two-player zero-sum board games.
    From the first versions of AlphaGo~\cite{Silver2016Mastering}, which initialized learning by relying on human expert positions, through AlphaGo Zero~\cite{silver2017mastering}, which used only self-play reinforcement learning, and AlphaZero \cite{silver2018general}, which followed the same approach to master Chess and Shogi too, to MuZero~\cite{schrittwieser2020mastering}, being able to learn environment dynamics and incorporate Atari games as well.
    
    A common issue when generalizing is a degradation in solution quality and computational performance.
    Such a barrier was one of the reasons the General Game Playing (GGP) based on Game Description Language (GDL) \cite{Genesereth2005General,Love2008General}, which initially flourished with many new achievements, started to lose popularity.
    The proposed logic-based system, although expressive and with a clean design, was computationally inefficient due to the forced logic resolution required to play games \cite{Sironi2016Optimizing}.
    Furthermore, it was difficult to encode games with complex rules, due to the need of implementing everything from scratch, ultimately leading to lengthy descriptions that were expensive for reasoning.
    This problem was even more evident in the case of GDL-II \cite{Thielscher2010AGeneral}, an extension of GDL designed to handle games with imperfect information and randomness.
    
    This issue sparked an arms race in designing new GGP formalisms, aiming to be fast, powerful in expressibility, and easy to use for game creation.
    The most advantageous in this regard were Ludii~\cite{Piette2020Ludii} and Regular Boardgames (RBG)~\cite{Kowalski2019RegularBoardgames}, which represent opposing principles.
    Ludii is a rich language that encodes many game features (called \emph{ludemes}) as parts of the language, which simplifies game descriptions and allows for the direct optimization of these features' implementation.
    RBG aims to be small, simple, and efficient; it is designed for board games, and its descriptions are compiled to C++.

    \subsection{Related Work}
        General Game Playing has been formed as a concretization of Artificial General Intelligence dedicated to games, following the belief that an intelligent agent should be able to adapt and successfully play any given game, not just one it was tailored to.
        The origins of this domain go back to \cite{Pitrat68}.
        Public attention and establishment of GGP as a fully-fledged research area begins with publishing GDL and starting the annual International General Game Playing Competition (2005--2016), initially co-located with the AAAI conference \cite{Genesereth2005General}.
        
        Typically, each GGP system introduces a domain-specific language that formally describes a family of games in a human- and machine-processable way, with explicit execution semantics.
        Then, to play a given game, the agents are either provided with its rules in this language or with its forward model, allowing them to simulate the game course.
        An alternative approach to GGP is to implement games in a standard programming language and use them in agents through a common interface.
        These generalizable approaches gain special attention from the Reinforcement Learning community, as they provide an excellent benchmarking domain.
        Examples of such generalized systems using game-specific implementations are OpenSpiel \cite{LanctotEtAl2019OpenSpiel}, Polygames \cite{Polygames}, GBG \cite{Konen2019GBG}, and Ai~Ai \cite{TavenerAiAi}.
        
        Most of the GGP languages support narrow classes of games, as they are designed to push research in these concrete areas. 
        Good examples are Video Game Description Language, representing Atari-like games \cite{Perez2016General}, and Ludi, which covers a specific subset of combinatorial games allowing easy procedural generation \cite{Browne2010Evolutionary}.
        Multiple languages exist for the generalization of chess-like games, including foremost METAGAME \cite{Pell1992METAGAME}, and Simplified Boardgames \cite{Bjornsson2012Learning}.
        
        On the other hand, a few GGP systems aimed for a real generalization and to describe universal domains.
        The aforementioned GDL is able to describe any turn-based, finite, deterministic, $n$-player game with simultaneous moves and perfect information \cite{Love2008General}.
        Its extension, GDL-II, also allows for expressing games with randomness and incomplete state knowledge \cite{Thielscher2010AGeneral}.
        Both languages are based on the standard syntax and semantics of Logic Programming, which makes language definition very minimal (based only on a few keywords) and convenient for many tasks, but as the required logic resolution is computationally expensive, it makes large games unplayable.
        
        Ludii is a GGP language created to encode all traditional strategy games throughout recorded human history as a part of the Digital Ludeme Project \cite{Piette2020Ludii}.
        It is based on high-level game-specific language constructions, and can encode finite non-deterministic and imperfect-information games \cite{Soemers2024LudiiUniversal}.
        Ludii incorporates game concepts directly into the language.
        Consequently, it becomes very complex with a few thousand keywords {\cite{LudiiLanguageReference}}.
        Yet, the game descriptions are relatively short, and it is easy to add new games or rule variants that use concepts already present in the language.
        Ludii has a large database of (primarily board) games and handles imperfect information and randomness.
        However, the involvement makes it usable only within the Ludii software developed in Java (which is source-available, though not open-source).
        It also aims to be efficient, containing many optimizations (e.g., bit-boarding), and is faster than the fastest GDL reasoners (based on propositional networks \cite{Sironi2016Optimizing}).
        
        Regular Boardgames (RBG) aims to be minimal, following the principle of GDL.
        It handles only perfect information deterministic rules (e.g., Chess, Draughts, Gomoku) and is primarily dedicated to board games, assuming the existence of one board and auxiliary arithmetic variables.
        The rules are encoded as a regular expression defining the language of possible game plays.
        RBG uses a compiled approach, where given rules are translated to a highly optimized reasoner module in C++, providing a common interface for traversing the game tree.
        In this matter, RBG provides the fastest automatically generated reasoners, yet complex games (like Chess) suffer from lengthy descriptions and long compilation times.
        RBG outperforms Ludii in terms of efficiency by an order of magnitude \cite{Kowalski2020EfficientReasoning} and so far is the fastest GGP language.

    \subsection{Our Contribution: Regular Games System}
        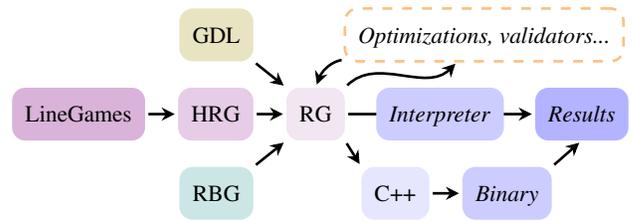
\begin{figure}
            \centering
            \resizebox{\linewidth}{!}{
                \begin{tikzpicture}[font=\scriptsize]
                    \tikzset{every node/.style={minimum height=16pt, inner sep=4pt, outer sep=0, rounded corners, text centered, text depth=0}}
                    \tikzset{arrow/.style={draw, -{Stealth[length=4pt,width=4pt]}, thick, shorten <= 1pt, shorten >= 1pt}}
            
                    \node [fill=violet!10] (rg) {RG};
        
                    \node [fill=olive!20,  above left=8pt and 10pt of rg] (gdl) {GDL};
                    \node [fill=violet!20,       left=        10pt of rg] (hrg) {HRG};
                    \node [fill=teal!20,   below left=8pt and 10pt of rg] (rbg) {RBG};
        
                    \node [fill=violet!30, left=10pt of hrg] (lg) {LineGames};
    
                    \path [arrow] (lg.east) -- (hrg.west);
        
                    \node [fill=blue!10, below right=8pt and  5pt of         rg] (cpp_source) {C++};
                    \node [fill=blue!20,       right=         10pt of cpp_source] (cpp_binary) {\textit{Binary}};
                    \node [fill=blue!30, above right=8pt and -5pt of cpp_binary] (results)    {\textit{Results}};
        
                    \path [arrow] (gdl.south east) -- (rg.north west);
                    \path [arrow] (hrg) -- (rg);
                    \path [arrow] (rbg.north east) -- (rg.south west);
            
                    \path [arrow] (rg.south east) -- (cpp_source.north west);
                    \path [arrow] (cpp_source) -- (cpp_binary);
                    \path [arrow] (cpp_binary.north east) -- (results.south);
                    \path [arrow] (rg) edge node[fill=blue!20] {\textit{Interpreter}} (results);
        
                    \node [draw=orange!50, dashed, thick, above right=8pt and 0pt of rg] (optimizations) {\textit{Optimizations, validators...}};
                    \draw [arrow] (rg.north east) to [out=45, in=-135] (optimizations);
                    \path [arrow] (optimizations) to [out=190, in=70] (rg.north);
                \end{tikzpicture}
            }
            \caption{Regular Games ecosystem.}
            \label{fig:ecosystem}
        \end{figure}
    
    We present a GGP system that combines the best features of the ones mentioned above.
    It consists of separate levels of abstraction, realized through more than just one language.

    \subsubsection{Low-level Language}
        \emph{Regular Games} (RG) is our base language, which is minimal and involves only a few simple elements and mechanisms.
        This makes it easy for automatic processing by, e.g., agents, optimizations, and rule analysis.
        
        The language is universal for the class of turn-based games with imperfect information and randomness.
        The rules are described by a nondeterministic finite automaton (NFA) (or a directed graph with edge labels).
        This representation is both more flexible and theoretically (exponentially) more succinct \cite{GH2008RegularExpressionSize} than regular expressions that RBG is based on.
        Following the GDL principle, RG does not have built-in concepts like a board, arithmetic, or predefined move forms.
        Instead, it operates only on abstract symbols (names) and compound types.
        Consequently, the representation of moves is also flexible, being an arbitrary sequence of symbols.
        
        RG's ecosystem (Fig.~\ref{fig:ecosystem}) includes an industry-grade code editor with LSP, automaton visualization, benchmarking tools, and a debugger of game description transformations.
        It includes an optimization pipeline that performs analysis, modifies the shape of the rules, and infers additional information for improving the efficiency of reasoning.
        
        Like RBG, RG is compiled into a C++ reasoner module which can be a part of any larger program (e.g., an agent or a match referee).
        
        \subsubsection{High-level Languages}
        
        The primary higher-level language is \emph{High-level Regular Games} (HRG) -- a more friendly language for game design, both by humans and automatically (via procedural content generation).
        HRG is equipped with numerous convenience mechanisms and describes the rules using typical declarative and structural programming constructs, while still remaining fully general, i.e., suitable for arbitrary games.
        It is effectively translated to low-level RG.
        
        Above the HRG, there are specialized frameworks for particular game types.
        This part adapts Ludii's approach by hardcoding common game concepts for being reused in similar games.
        We developed an example of such a framework, supporting Alquerque-like games on line boards.
        Instead of developing original syntax, this framework is a Python library that allows defining game rules in a few lines.
        It generates HRG code, which can be further modified to include less standard elements if necessary.

        We also developed automatic translations of RBG to RG and GDL to RG (experimental).
        In this way, RG can serve as the target language for many high-level languages, providing them efficiency, uniform target representation, and tooling.
        
        The efficiency of RG outperforms RBG in all games implemented in HRG (handmade or generated) and in part of the games automatically translated from RBG.
        Consequently, RG is also typically 10--20 times faster than Ludii.

\section{Regular Games Language}
    \subsection{Full Definition}
        A \emph{regular game} description consists of a set of \emph{type aliases}, a set of \emph{variables}, a set of \emph{constants}, and a list of \emph{edges} defining the \emph{rules automaton}.
        The order of elements is arbitrary.
    
        Every \emph{name} (of e.g., variable) is a case-sensitive string consisting of alphanumeric characters and an underscore without leading digits.
        There are three names reserved for language keywords: \RG!const!, \RG!type!, and \RG!var!.
    
        \subsubsection{Types and values}
            A \emph{type} is either a finite \emph{set type}, whose domain contains \emph{symbols} (denoted as \RG!{s_1,...,s_n}!), or an \emph{arrow type}, whose domain contains \emph{maps} where keys are of the source type and the values are of the destination type (denoted as \RG!t_1->t_2!, where \RG!t_1! and \RG!t_2! are the source and destination types respectively).
            Only set types can be used as the source in arrow types.
    
            We define type \emph{equality} (denoted as \RG!T=U!).
            Set types are equal when their domains are equal.
            Map types are equal when their source and destination types are equal.
    
            We define type \emph{assignability} (denoted as \RG!U:=T!).
            Set types are assignable if their domains have at least one common element.
            Map types \RG!T_1->T_2! and \RG!U_1->U_2! are assignable if \RG!T_1:=U_1! and \RG!T_2:=U_2!.
            Note that both \RG!=! and \RG!:=! are commutative.
    
            A \emph{value} is either a \emph{symbol} (denoted as \RG!s!) or a \emph{map} of key-value pairs supplied with a default value (denoted as \RG!{:v,s_1:v_1,...,s_n:v_n}!, where \RG!k_i! are symbols used as keys, \RG!v_i! are values, and \RG!v! is the default value).
            For all non-specified keys, the map contains the default value.
    
        \subsubsection{Type aliases, constants, and variables}
            A \emph{type alias} allows reusing types and can be used interchangeably with the type they refer to (denoted as \RG!type T=t;!, where \RG!T! is the name, \RG!t! is the type it aliases).
            Recursive type aliases are forbidden.
    
            \begin{lstlisting}[style=RG, gobble=16]
                type Coord = {0, 1, 2};
                type Piece = {e, X, O};
                type ColumnOfBoard = Coord -> Piece;
                type Board = Coord -> ColumnOfBoard;
            \end{lstlisting}
    
            A \emph{constant} is an invariable value of some type (denoted as \RG!const C:T=v;!, where \RG!C! is the name, \RG!T! is a type, and \RG!v! is the value).
            Constants can reference other constants, though recursive values are forbidden.
    
            \begin{lstlisting}[style=RG, gobble=16]
                const next: Coord -> Coord = {0:1, 1:2, :0};
                const initColumn: ColumnOfBoard = {:e};
                const initBoard: Board = {:initColumn};
            \end{lstlisting}
    
            A \emph{variable} is a container for a value of some type that can change between the game states (denoted as \RG!var V:T=v;!, where \RG!V! is the name, \RG!T! is the type, and \RG!v! is the initial value).
    
            \begin{lstlisting}[style=RG, gobble=16]
                var posX: Coord = 0;
                var board: Board = initBoard;
            \end{lstlisting}
    
        \subsubsection{Expressions}
            An \emph{expression} defines a way to evaluate a value given a \emph{variables assignment} (mapping from variable names to their values).
            There are three types of expressions:
            \\$\bullet$\;A \emph{reference} to a constant, a variable, or a symbol (denoted by its name).
            Its type is inferred from the referenced value except for the symbol, whose inferred type is \RG!{s}! (a singleton set type of itself).
            \\$\bullet$\;An \emph{access} is a compound of two expressions (denoted as \RG!e_1[e_2]!).
            The type of the access expression is the destination type of the type of \RG!e_1!.
            In a proper game, \RG!e_1! is of an arrow type, and the type of \RG!e_2! is assignable to \RG!e_1!'s source type.
            \\$\bullet$\;A \emph{cast} is a compound of a type and an expression (denoted as \RG!T(e)!).
            The type of a cast expression is \RG!T!.
            Map types are cast recursively -- if \RG!T=T_1->T_2!, then \RG!e!'s keys are cast to \RG!T_1! and \RG!e!'s values are cast to \RG!T_2!.
            In a proper game, the type of \RG!e! is assignable to \RG!T!, as well as that all cast symbols belong to the set types they were cast to.
    
            \begin{lstlisting}[style=RG, gobble=16]
                board                 // Board
                board[posX]           // ColumnOfBoard
                board[Coord(posX)][1] // Piece
            \end{lstlisting}
    
        \subsubsection{Rules automaton}\label{subsec:rulesautomaton}
            A \emph{rules automaton} is a pair $(Q,\delta)$, where $Q$ is the (finite) set of \emph{nodes} and $\delta\colon Q \times \mathit{Actions} \to Q$ is the transition function, and $\mathit{Actions}$ is the set of possible actions defined later.
            The elements of $Q$ are states of the automaton, which are called \emph{nodes} to avoid confusion with game states.
            $Q$ always contains \RG!begin! and \RG!end! nodes that have a special meaning.
        
            \subsubsection{Game state}
                A \emph{game semistate} $\mathcal{S}$ is an assignment for all variables.
                The \emph{initial game semistate} $\mathcal{S}_\mathrm{I}$ is the one where all variables take the initial values from their definitions.
    
                A \emph{game state} is a pair $(\mathcal{S},q)$, where $\mathcal{S}$ is a game semistate and $q \in Q$ is a node.
                The \emph{initial game state} is $(\mathcal{S}_\mathrm{I},$\,\RG!begin!$)$.
    
            \subsubsection{Paths}
                A \emph{node} is an arbitrary name.
                A \emph{transition} in $\delta$ is a $3$-tuple $(q_1,q_2,a)$, where $q_1,q_2 \in Q$, and $a$ is an \emph{action}, which labels the transition.
    
                For a game semistate $\mathcal{S}$, an action $a$ can be \emph{legal} or not.
                A legal action can be \emph{applied}, which results in a modified game semistate denoted as $\mathcal{S}\cdot a$.
    
                For a game state $(S,q)$, a transition $(q_1,q_2,a)$ is \emph{legal} if $q=q_1$ and $a$ is legal for $S$.
                Then the resulting game state is $(S\cdot a,q_2)$.
                We also say that the action $a$ is \emph{legal} for the semistate $S$.
                A \emph{legal walk} for $(S,q)$ is a sequence of legal transitions (edges) for consecutively resulting game states.
                
                An action is \emph{valid} for $S$ if its legality can be correctly computed.
                The legality of invalid actions is undefined.
                For the purpose of efficiency, the validity of actions does not have to be checked, but it should be guaranteed in a proper description (defined later).
    
            \subsubsection{Actions}
                There are five types of actions.
                \\$\bullet$\;The \emph{empty action} (no denotation).
                It is always legal and valid, and does not affect the semistate.
                \begin{lstlisting}[style=RG, gobble=20]
                    q1, q2:;
                \end{lstlisting}
                $\bullet$\;A \emph{comparison}, which is either an equality or an inequality of two expressions \RG!e_1! and \RG!e_2! (denoted as \RG!e_1==e_2! and \RG|e_1!=e_2| respectively).
                It is legal if the resulting values from the evaluated expressions are the same or distinct, respectively.
                It is valid if the type of \RG!e_1! is assignable to the type of \RG!e_2!.
                \begin{lstlisting}[style=RG, gobble=20]
                    q1, q2: board[posX][1] == e;
                    r1, r2: next[posX] != 2;
                \end{lstlisting}
                $\bullet$\;An \emph{assignment} (denoted as \RG!e_1=e_2!), which sets the value represented by \RG!e_1! to the value evaluated from \RG!e_2!.
                It is legal and valid if the type of \RG!e_2! is assignable to the type of \RG!e_1!, as well as all assigned symbols belong to the set types they were assigned to.
                \begin{lstlisting}[style=RG, gobble=20]
                    q1, q2: board[posX][1] = x;
                    r1, r2: posX = next[posX];
                \end{lstlisting}
                $\bullet$\;A \emph{reachability check}, which verifies a complex condition.
                The reachability check specifies two nodes \RG!q1!, \RG!q2! $\subseteq Q$ and is denoted as \RG!?q1->q2! or \RG{!q1->q2}.
                In the first case, the action is legal for a semistate $S$, if there exists a legal walk from $(S,$\RG!q1!$)$ to $(S',$\RG!q2!$)$ for some semistate $S'$.
                The second case is the negation; hence, it is legal if such a legal walk does not exist.

                The following example is taken from Tic-Tac-Toe:
                \begin{lstlisting}[style=RG, gobble=20]
                    p1, p2: ? q1 -> q2; // Any empty square?
                    r1, r2: ! q1 -> q2; // No empty squares?
                    q1, q2: board[0][0] == e;
                    q1, q2: board[0][1] == e;
                    // ...
                    q1, q2: board[2][2] == e;
                \end{lstlisting}
                A reachability check outgoing from a node \RG!p1! and with a starting node \RG!q1! is valid for $S$ only if there is no legal walk from $(S,$\RG!q1!$)$ to a game state with \RG!p1!.
                Therefore, recursion is forbidden.
                \\$\bullet$\;A \textit{tag} (denoted as \RG!DOLLARs!, where \RG!s! is the tag's name).
                It is always legal and valid, and becomes a part of a \textit{move} (defined below).
                \begin{lstlisting}[style=RG, gobble=24]
                    q1, q2: DOLLAR 1;
                \end{lstlisting}
    
        \subsubsection{Special definitions}
            A handful of definitions have a special meaning and must be present in every regular game in the form specified below.
            \\$\bullet$\;\RG!keeper! and \RG!random! are symbols representing two special players.
            The former is managing the game; the latter is used to handle nondeterministic moves.
            \\$\bullet$\;\RG!begin! is a dedicated automaton state for the initial state.
            \\$\bullet$\;\RG!end! is a dedicated automaton state that, once entered, ends the game.
            In a proper game, it is always entered with a \RG!player=keeper! assignment.
            \\$\bullet$\;\RG!type Player! is a set type representing the players of the game.
            Neither \RG!keeper! nor \RG!random! are included.
            \\$\bullet$\;\RG!type Score! is a set type representing the possible outcomes of the players (usually natural numbers).

            Some definitions are built-in and can be omitted.
            When provided, they have to match the definitions below.
            \\$\bullet$\;\RG!type Goals=Player->Score! and the corresponding \RG!var goals:Goals! specifies the \emph{outcomes} of each player.
            The initial value can be chosen arbitrarily and defaults to the first symbol of \RG!Score!.
            \\$\bullet$\;\RG!type PlayerOrSystem! must be precisely the union of \RG!type Player! and \RG!{keeper,random}!.
            \\$\bullet$\;\RG!var player:PlayerOrSystem=keeper! specifies the current player.
            The \RG!keeper! always begins and ends the game.
            \\$\bullet$\;\RG!type Bool={0, 1}!.
            \\$\bullet$\;\RG!type Visibility=Players->Bool! and the corresponding \RG!var visible:Visibility! specifies for every player whether the current part of the play is visible to them.
            (The symbol \RG!1! means that it is visible.)
            The initial value can be chosen arbitrarily and defaults to \RG!1!.
    
            As most of the above are implicit, a minimal regular game description can be as follows:
    
            \begin{lstlisting}[style=RG, gobble=16]
                type Player = {x};
                type Score = {0};
                begin, end: player = keeper;
            \end{lstlisting}
    
        \subsubsection{Moves and plays}
            To build a play, the players choose their moves for the current game state.
            For a game state $(\mathcal{S},q)$, a \emph{move walk} $P$ is a legal walk whose last transition is an assignment to \RG!player! variable, and there are no other such assignments in it.
            
            Players do not specify particular move walks, but their labelings.
            A \emph{labeling} of a legal walk $P$ is the sequence of tags that occurs on the transitions of $P$, given in the same order.
            A \emph{move} is a finite sequence of names, and it is \emph{legal} if it is the labeling of a move walk.
            For the same move, there can be many move walks $P$, but in a proper game, the effect of every move walk with the same labelling must be the same, i.e., for a game state $(\mathcal{S},q)$ and a legal move $M$, every legal walk $P$ labeled by $M$ leads to the same next game state $(\mathcal{S}',q')$.
            
            A \emph{play} is the concatenation of legal moves applied sequentially to the initial game state.
            A play is \emph{completed} if it finally leads to a game state whose automaton node is \mbox{\RG!end!}.
            When the play is completed, the players' \emph{outcomes} are the values in \RG!goals!.
            A play is built in the way that the player currently assigned to \RG!player! chooses its legal move.
            There are two special \emph{system players}, which are executed by a game manager:
            \\$\bullet$\;\RG!keeper! -- it chooses any move.
            A proper game description must ensure that it always has exactly one move.
            The keeper is used to perform modifications of the game state that are not assigned to a particular real player.
            It plays a vital role in improving the efficiency, and in games with imperfect information, it can be used to reveal partial information to players.
            Typically, its move is empty (no tags).
            \\$\bullet$\;\RG!random! -- it chooses a uniformly random move from the legal ones and is used to introduce randomness into the game.
            The probability distribution can be controlled by multiplying moves and/or through a sequence of moves.
    
            In games with imperfect information, players do not necessarily see the whole moves of the others.
            An \emph{obfuscated move} for a player \RG{p} is a subsequence of a move where all the labels of the edges, such that \RG!visible[p]==0! at the moment of traversal, are removed.
            After every move of a player (including the system players), all the other players are informed of obfuscated moves for them.
    
            \subsubsection{Proper description}
                A regular game to be \emph{proper} must satisfy several conditions.
                The first condition must hold for every game state $(S,q)$ such that there is a legal walk from the initial state $(\mathcal{S}_\mathrm{I},$\,\RG!begin!$)$ by a sequence of legal actions:
                \begin{enumerate}
                    \item[1.] (Action validity) For every outgoing edge from $q$, all actions must be valid for $S$.
                    This condition ensures the behavior of actions is well defined for every game state that we may encounter during a computation.
                \end{enumerate}
                The other conditions hold for plays.
                For every play $P$, where $(S,q)$ is obtained by a legal walk labeled by $P$ from the initial game state, the following hold:
                \begin{enumerate}
                    \item[2.] (Move unambiguosity) For every move $M$ for $(S,q)$, every legal walk from $(S,q)$ labeled by $M$ leads to the same next state $(S',q')$.
                    This condition implies that every play uniquely defines a game state obtained by applying a legal walk labeled by this play.
                    \item[3.] (Continuable) If the play is not complete, then there exists at least one legal move.
                    \item[4.] (Deterministic keeper) If \RG!keeper! is on the move at $(S,q)$, then it has exactly one legal move.
                    \item[5.] (Game finiteness) There exists a finite number of plays. This ensures that there is no cycle in the game states, hence the game cannot be played infinitely.
                    Together with condition~(3), it also implies that finally every play can be extended to a complete play.
                \end{enumerate}
    
        \subsubsection{Shorthand Actions}
            In addition to the actions described in section~\ref{subsec:rulesautomaton}, there are two more that are used solely to shorten game descriptions.
            Both are optional and can be automatically expanded using their corresponding transformations (see section~\ref{subsec:transformations}).
            \\$\bullet$\;An \emph{any assignment} (denoted as \RG!e=t(*)!), which sets the value represented by \RG!e! to any value of type \RG!t!.
            It is equivalent to multiple parallel edges with a standard assignment on each.
            Therefore, in a proper game, \RG!t! is a set type, assignable to the type of \RG!e!, as well as all assigned symbols belong to the set types they were assigned to.
            \begin{lstlisting}[style=RG, gobble=16]
                // Shorthand              // Expanded
                q1, q2: posX = Coord(*);  q1, q2: posX = 0;
                                          q1, q2: posX = 1;
                                          q2, q2: posX = 2;
            \end{lstlisting}
            $\bullet$\;A \emph{variable tag} (denoted as \RG!DOLLARDOLLAR V!), which yields a tag with the current value of variable \RG!V!.
            It is equivalent to multiple parallel paths with a comparison and a tag each.
            Therefore, in a proper game, type of \RG!V! is a set type.
            \begin{lstlisting}[style=RG, gobble=16]
                // Shorthand
                q1, q2: DOLLARDOLLAR posX;
                // Expanded: p0, p1, and p2 are new nodes
                q1, p0: posX == 0; p0, q2: DOLLAR 0;
                q1, p1: posX == 1; p1, q2: DOLLAR 1;
                q1, p2: posX == 2; p2, q2: DOLLAR 2;
            \end{lstlisting}
        
        \subsubsection{Pragmas}
            Optionally, every regular game can be annotated with a list of \emph{pragmas}.
            We consider them to be implementation-specific, and their only purpose is to hint the runtime to allow faster execution.
            Pragmas do not affect the game tree in any way (it is the same as without them) and can be safely ignored.
            All pragmas start with a \RG!@! and end with a \RG!;!, which makes them trivial to ignore.
    
            \begin{lstlisting}[style=RG, gobble=16]
                // At most one outgoing action is legal.
                @disjoint p1 : q1 q2 q3;
            \end{lstlisting}
    
    \subsection{Theoretical Expressivennes}
        RG can encode any finite turn-based game with imperfect information and randomness.
        Since simultaneous moves can be modelled using imperfect information, this is the same class as for GDL-II and Ludii.
        
        \begin{theorem}
            Regular Games Language is universal for the class of all finite turn-based games, including imperfect information and with randomness, where probabilities are rational numbers.
        \end{theorem}
        \begin{proof}[Proof idea]
            The proof provides a reduction from a game in this class given in the \emph{extensive form} \cite{Rasmusen1994Games} to an equivalent RG description.
            Full proof in the Appendix.
        \end{proof}
        
        The theoretical computational complexity depends on the succinctness of game states.
        Let the \emph{type length} be the number of arrows plus one in the type definition.
        Set types have type length $1$.
        If the type length is fixed, game states have a polynomial size.
        
        In general, the size of a game state can be exponential in the length of the game description.
        Computing one move can be as hard as traversing the whole game tree.
        Consequently, the representative problem of determining if a player has a legal move and all the problems verifying the conditions of a proper description are EXPSPACE-complete.
        
        \begin{theorem}
            Given a game description in Regular Games, the problem of deciding whether from the initial game state there is a legal move is EXPSPACE-complete.
            The same holds for verifying conditions (1)--(5) of a proper game description.
            If the maximum type length is fixed, then these problems become PSPACE-complete.
        \end{theorem}
        \begin{proof}[Proof idea]
        The hardness proofs reduce from a canonical problem with a Turing machine with exponential/polynomial tape.
        The membership proofs use nondeterminism and a logarithmic counter for counting (doubly-)exponentially many game states. Full proof in the Appendix.
        \end{proof}

\section{Language Ecosystem and Tooling}
    \subsection{High-level Regular Games}
        While very simple, RG is also verbose -- especially for large maps and the list of edges.
        To alleviate that, High-level Regular Games (HRG) takes a different approach -- it defines the automaton using syntax based on popular (structural) programming languages.

        \begin{lstlisting}[style=HRG, gobble=12]
            // Excerpt from Tic Tac Toe.
            domain Player = x | o
            domain Piece = empty | Player
            domain Position = P(I, J)
                where I in 0..2, J in 0..2

            board : Position -> Piece = { P(I, J) = empty
                where I in 0..2, J in 0..2 }
            next_vertical : Position -> Position
            next_vertical(P(I, J)) = P((I + 1) % 3, J)

            reusable graph anyEmpty() {
              forall position: Position {
                check(board[position] == empty) } }

            graph turn(me: Player) {
              player = me
              // ...
              if not(reachable(anyEmpty())) { end() } }

            graph rules() { loop { turn(x) turn(o) } }
        \end{lstlisting}

        Non-automaton definitions are also more natural -- numeric ranges and set unions help with repetition, and pattern matching improves readability while simplifying the exhaustiveness check (i.e., whether all cases are covered).

    \subsection{Compatibility with Other Languages}
        On top of HRG, one can define languages and generators dedicated to particular game classes.
        As an example, we developed \emph{LineGames} generator, where we defined 23 unique existing games in a concise way.
        It is implemented as a Python library, taking advantage of a well-known syntax and compatibility; yet, it would also be trivial to develop a dedicated textual format.
        Games defined in this way apply the same optimization pipeline as those written directly in HRG, hence do not come with any performance penalty.
        
        Another developed translation is RBG to RG.
        It follows Thompson's construction, transforming the regular expression to an automaton.
        The RBG's specific concepts (board, arithmetic, position, piece counters) are embedded via general language constructs.
        Semantic differences (e.g., in RBG, the game ends when a player has no legal move) are handled in a post-processing step.
        The resulting efficiency is comparable and sometimes beats the RBG system itself.
        
        There is also a currently experimental translation from GDL to RG, based on the propositional network's approach.

    \subsection{Transformations and Validators}\label{subsec:transformations}
        Games automatically translated to RG are usually not in their most performant form. 
        We developed a number of transformations that discover certain patterns and simplify them, while preserving the correctness of the game.

        Some of these optimizations are specific for RG (e.g., pruning redundant tags), while others are also used for general-purpose programming languages (e.g., inlining assignments or constant propagation).
        As applying one transformation can enable others, they are run in a fixed-point loop, in an interaction-based order.
        While most transformations are local to a fragment of the automaton, some require a global view of the game.
        To achieve this, we employ data-flow analysis~\cite{Kildall1973Optimization}, which calculates information about the state of the game (knowledge) at each node. 
        
        To ensure game description correctness at all times, we run a series of validators after every applied transformation.
        This includes type checking, reachability possibleness (i.e., whether the automaton remains connected), and map correctness (i.e., all maps have unique keys and a default value).

        Optimized automata translated from RBG have over 72\% fewer nodes, 66\% fewer edges, and even 21\% smaller \emph{state size} -- memory needed to store the game state, which is a key factor for the reasoner's efficiency.
        For games written in HRG, optimizations reduce the number of nodes and edges respectively by 47\% and 41\%.

        The complete translation and optimization suite for most HRG games runs under 100ms, resulting in an almost immediate feedback loop in the IDE.

    \subsection{Programming Interface of Forward Model}
        The compiler takes an Abstract Syntax Tree (AST) of a game description in RG and generates C++ code providing a forward model for the game.
        C++ was chosen as a target language because it is one of the most effective and commonly used languages, and due to its interoperability with other languages.
        The generated module provides functions for traversing the game tree: checking if a game state is terminal, computing all legal moves, applying a move, etc. 
        This enables writing AI agents and testing them generically on all available games.
        
        Generating legal moves is based on automaton traversal, by applying legal actions and collecting all legal moves.
        Applying a move works analogously: the game state transition is performed through a legal walk corresponding to the given move.
        The compiler applies the preceding analysis (through pragmas) to produce the most efficient model for the game.
        
        The following is an example part of the main interface:
        \lstset{language=C++,
        keywordstyle=\color{blue}\ttfamily,
        commentstyle=\color{magenta}\ttfamily,
        morecomment=[l][\color{magenta}]{\#},
        }
        \begin{lstlisting}[language=C++, gobble=12]
            class GameState {
             ...
             // Is the current node `end`
             bool isTerminal() const;
             // Gets the player on move
             PlayerOrSystem getCurrentPlayer()const;
             // Fills the vector with all legal moves
             void getAllMoves(
               std::vector<Move>& moves, Cache& cache);
             // Applies the given legal move
             void apply(const Move& move, Cache& cache);
             // Returns a human-readable representation
             std::string getStateDescription()const;
            }
        \end{lstlisting}

    \subsection{IDE}
        \begin{figure}
            \centering
            \includegraphics[width=\linewidth]{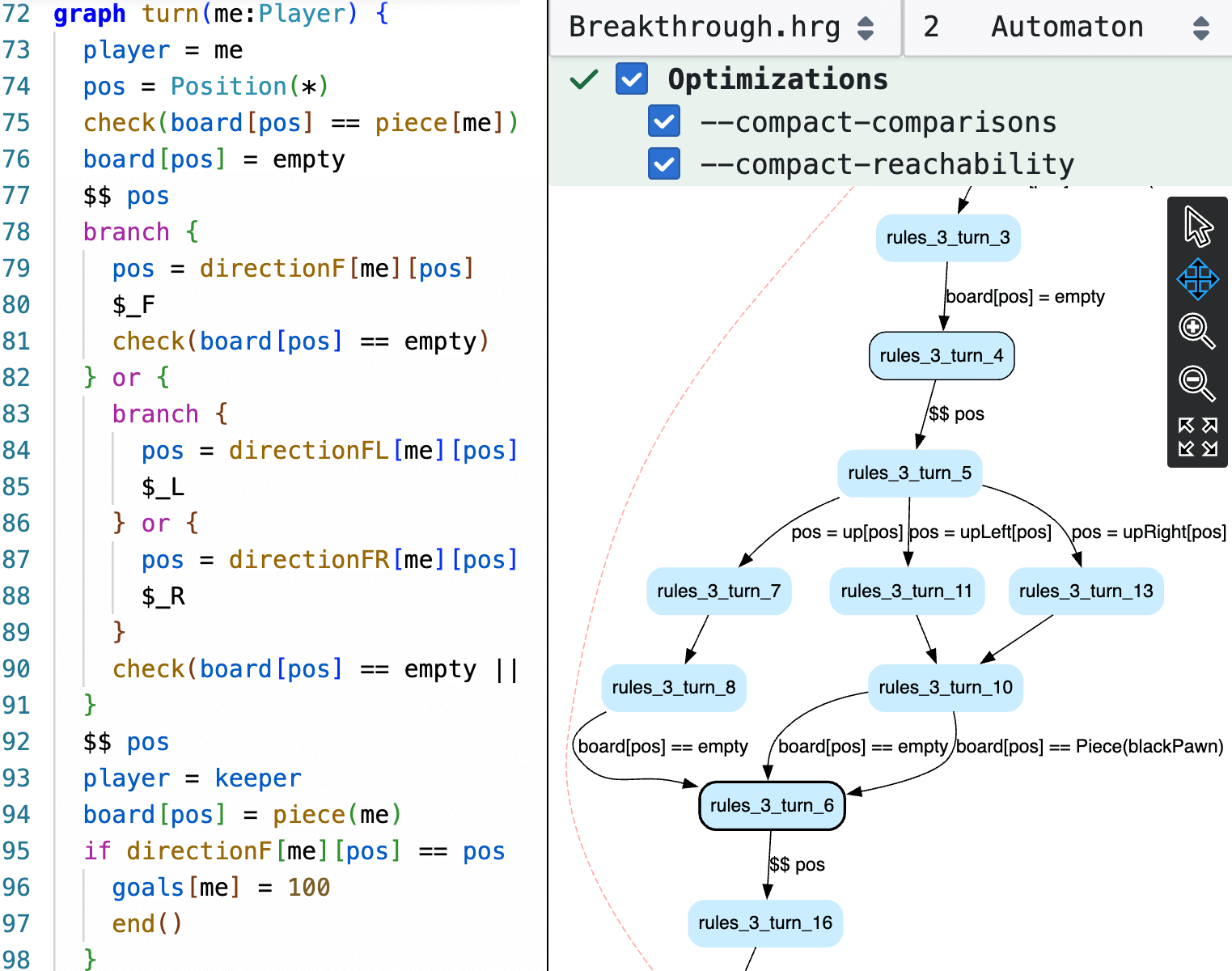}
            \caption{
                The IDE is split into two parts.
                The left side contains the code editor, while the right side includes benchmarking tools, automaton visualization, transformation configuration, and a snapshot of the game description after every transformation.
                It is available online at Anonymized.}
            \label{fig:ide}
        \end{figure}
    
        RG and Ludii are the only GGP languages with a proper, industry-grade Integrated Development Environment (IDE).
        The main part of RG's IDE, presented in Fig.~\ref{fig:ide}, is the code editor with support for Language-Server Protocol (LSP).

        Thanks to LSP, the code editor provides game developers with features they know from the code editor they use daily, such as navigation, syntax highlighting, auto-completion, and diagnostics.
        The IDE provides LSP features for HRG and RG, while basic editor functionalities are also available for both GDL and RBG.

\section{Efficiency Experiments}
    Table~\ref{tab:experiments} shows a comparison of the efficiency of game descriptions written in HRG (handmade or generated), automatically translated RBG to RG, and native RBG and Ludii using their own tooling.
    Due to the diversity in existing game variants, we attempted to align the rules with those of other systems.
    Hence, some games have two variants (RBG and Ludii), which differ mainly by turn limits.
    
    All games implemented in HRG yield a faster forward modal than both RBG and Ludii.
    The advantage over RBG comes from a more general and expressive language (allowing an implementation without workarounds, e.g., rotations in Pentago) and more advanced analysis (e.g., simplifying expressions, inlining assignments, detecting disjoint paths).
    
    \begin{table}[htb!]\centering\small
        \newcommand{\ccol}[1]{\multicolumn{1}{c}{#1}}
        \newcommand{\LG}{}
        \caption{Flat Monte Carlo playouts per second on a range of games that are also available in other GGP systems (RBG and Ludii).
        Results are averaged over three 1-minute runs.
        }\label{tab:experiments}
        \begin{tabular}{l@{\hskip-6pt}rrrr}
            \toprule
                                     & \ccol{\textbf{HRG}} & \ccol{\textbf{RBG}} & & \\
                                     & \ccol{$\downarrow$} & \ccol{$\downarrow$} & & \\
            \textbf{Game}            & \ccol{\textbf{RG}}  & \ccol{\textbf{RG}}  & \ccol{\textbf{RBG}} & \ccol{\textbf{Ludii}} \\\midrule
            Alquerque (lud)          &\LG 24 871           & 16 510              & 15 962              & 1 981 \\
            Alquerque (rbg)          &\LG 116 965          & 79 899              & 79 312              & --    \\
            Amazons                  & 6 226               & 3 582               & 3 693               & --    \\
            Amazons (split2)         & 35 222              & 24 116              & 30 365              & 3 199 \\
            Ataxx                    & 11 020              & --                  & --                  & 555   \\
            Backgammon               & 2 323               & --                  & --                  & 7${}^\dagger$ \\
            Bombardment              & 416 266             & --                  & --                  & 14 907 \\
            Breakthrough             & 82 135              & 79 175              & 50 977              & 3 365  \\
            Chess                    & 1 572               & 531                 & 995                 & 113${}^\dagger$ \\
            Chess (king capture)     & 13 170              & 4 887               & 11 062              & --     \\
            Clobber                  & 40 012              & --                  & --                  & 1 664  \\
            Connect Four             & 1 297 176           & 122 716             & 914 514             & 55 858 \\
            Dash Guti (lud)          &\LG 38 455           & --                  & --                  & 2 885  \\
            Dash Guti (rbg)          &\LG 95 419           & 64 334              & 74 039              & --     \\
            Dots and Boxes           & 36 362              & --                  & --                  & 1 993  \\
            English Draughts         & 69 244              & 45 927              & 62 251              & 4 104${}^\dagger$ \\
            Fox and Geese (lud)      & 20 153              & --                  & --                  & 1 457  \\
            Fox and Hounds           & 444 243             & 323 068             & 331 884             & 14 216 \\
            Gol Ekuish (lud)         &\LG 11 490           & --                  & --                  & 770    \\
            Gol Skuish (rbg)         &\LG 54 063           & 41 004              & 34 239              & --     \\
            Gomoku (standard)        & 26 126              & 23 862              & 22 458              & 19 603 \\
            Hex (11x11)              & 23 535              & 2 447               & 22 441              & 11 024 \\
            Knightthrough            & 143 966             & 81 870              & 86 355              & 2 191  \\
            Lau Kata Kati (lud)      &\LG 30 224           & --                  & --                  & 3 676  \\
            Lau Kata Kati (rbg)      &\LG 106 508          & 74 495              & 85 160              & --     \\
            Oware                    & 27 991              & --                  & --                  & 347    \\
            Pentago                  & 43 875              & 500                 & 22 553              & --     \\
            Pentago (split) (lud)    & 172 626             & 6 874               & 61 878              & 3 933  \\
            Pretwa (lud)             &\LG 40 084           & --                  & --                  & 5 401  \\
            Pretwa (rbg)             &\LG 273 431          & 176 254             & 167 237             & --     \\
            Reversi                  & 28 445              & 2 249               & 19 838              & 1 497  \\
            Surakarta                & 6 589               & 3 095               & 5 885               & 843    \\
            The Mill Game            & 61 514              & 31 403              & 43 116              & --     \\
            The Mill Game (lud)      & 35 674              & 17 143              & 24 563              & 2 667${}^\dagger$ \\
            Tic-Tac-Die              & 2 708 648           & --                  & --                  & 36 702 \\
            Ult. Tic-Tac-Toe         & 241 088             & --                  & --                  & 8 090  \\
            Yavalath                 & 415 251             & 359 832             & 352 910             & 93 642 \\
            \bottomrule
        \end{tabular}%
        \vspace{.2cm}
        {\raggedright\emph{Notes:} Some games exist in many variants concerning e.g., split moves, (non)mandatory captures, different turn limits.\\
        \quad(rbg) -- defined to comply exactly with the RBG variant.\\
        \quad(lud) -- defined to comply with the default Ludii variant.\\
        \quad$\dagger$ -- The Ludii variant contains known subtle differences, but we consider them minor enough or in favor with respect to efficiency (e.g., Backgammon contains a bug of sometimes preliminary ending a move; Chess executes choice of a promoted piece in a separate move; The Mill Game contains a bug that 3 pieces left cannot form a capturing mill).\\
        \emph{Environment:} AMD~Ryzen~9~3950X, 64GB, Ubuntu~24.04.3~LTS, g++~14.2.0, GraalVM~25.0.1+8.1.\\}
    \end{table}
    
\section{Conclusions}
    We have presented \emph{Regular Games}, a new GGP formalism that is as universal as GDL-II and Ludii while being simpler and much more computationally efficient.
    It beats RBG for all implemented games, the fastest GGP language so far, and has the potential to supersede it.
    
    The Regular Games system comes with a convenient programming interface that allows writing AI agents and testing them on the available games.
    Reliable research requires rigorous evaluation of algorithms across a diverse set of games.
    Many enhancements of algorithms like MCTS or Minimax do not work in all cases, so the subdomain where they are effective needs to be found and properly described.
    The efficiency becomes even more crucial for experiments with Reinforcement Learning approaches.
    
    Another novelty is the multilevel translation approach and the extensibility of RG by design.
    Usually, GGP languages are incompatible with any other, relying only on their own separate database of games.
    They are extensible in a limited way, by adding new mechanisms/keywords to the language, thus increasing the overall complication level and affecting the underlying engine.
    We have shown that our language can serve as an assembler for GGP, enabling translation of other description languages to RG instead of relying on their own executables.
    Developing new specialized languages does not require any change in the bottom pipeline.
    This makes the system excellent for procedural content generation applications and exploration of new game rules.
    
    More tools supporting the agent programmers for RG are still in development, but the system can already be used as a research platform.
    RG provides a C++ library that can be included and, for a given game, directly call all the functions of the forward model, giving the programmer full control of how to manage the communication with the game engine.
    
    Future work involves the development of other higher-level languages for, e.g., fairy chess, card, and dice games.
    Independently, the efficiency can be improved by extending the automaton analysis and applying more techniques (e.g., bit-boarding).
    The experimental GDL translation could also be improved, e.g., by detecting alternating moves.

\section{Acknowledgements}

This work was supported by the National Science Centre, Poland, under project number 2021/41/B/ST6/03691.

\bibliography{bibliography}

\appendix

\onecolumn
\pagestyle{plain}

\newpage
\section{Theoretical Expressiveness}

\begin{theorem}
    Regular Games is universal for the class of all finite turn-based games, including perfect or imperfect information, and with randomess where probabilities are rational numbers.
\end{theorem}
\begin{proof}        
Following \cite{Rasmusen1994Games}, we recall finite extensive-form games.
\begin{definition}
A finite extensive-form game is a structure $\mathcal{G} = (n, \mathcal{T}, \iota, \rho_s, u, \mathcal{I})$, where:
\begin{itemize}
\item $n$ is the number of players, which we name $P_1,\ldots,P_n$.
\item $\mathcal{T}$ is a finite tree with a set of \emph{states} $S$, unique initial state $s_0\in S$, set of terminal states $T\subseteq S$, and a predecessor function $p: (S\setminus\{s_0\})\rightarrow S$;
\item $\iota: (S\setminus T)\rightarrow \{\text{Nature},P_1,\ldots,P_n\}$ is a function indicating which state belongs to which player ($\text{Nature}$ player perfor ms random moves).
\item $\rho_s$ is a probability measure, defined for all states $s$ for which $\iota(s)=\text{Nature}$; 
\item $\mathcal{I}=\{\mathcal{I}_r | r \in \{ 1, \ldots, n \}\}$ is an information partition such that for each $I\in\mathcal{I}_r$ (information set):
\begin{itemize}
\item the elements in an $I$ are pairwise disjoint and sum up to $S$;
\item all children of a state $s$ such that $\iota(s)=r$ are in different information sets;
\item for all branches $B$ there is exactly one $s\in B$ such that $s$ occurs in some information set in $I$;
\item for all $i\in I$, the predecessors of all states in $i$ are in one $i'\in I$.
\end{itemize}
\end{itemize}  
\end{definition}

We construct a Regular Games game description so that the game tree is directly represented by the rules automaton.
We do not need any variables except the implicit ones: \RG{player}, \RG{visible}, and \RG{goals}.
Initially, \RG{visible} is set to $0$ for all players.

For each game state $s$ of $\mathcal{G}$, we add one node to the automaton, say $q_s$.
Furthermore, from \RG{begin} correspond to the initial state $s_0$ of $\mathcal{G}$, so \RG{begin}$=q_{s_0}$.
If $s$ is a terminal state, then we add a path from $q_s$ to \RG{end}, where \RG{goals} are set for each player according to the the utility function $u$ of $\mathcal{G}$; these edges will be executed by the keeper, so all incoming edges to $q_s$ have the assignment \RG{player=keeper}.

To handle imperfect information, we use tags to reveal partial information about the current state for each player separately.
Let $t_{r,s'}$ be unique for the information set of $P_r$ containing $s'$.
If $\iota(s) \neq \text{Nature}$, then for each child $s'$ of $s$, we add the following path from $q_s$ to $q_{s'}$:\\
\begin{enumerate}
\item[(1)] \RG{visible[P_$\iota(s)$]=1;}\\\$ $t_{\iota(s),s'}$;\\\RG{visible[P_$\iota(s)$]=0;}
\item[(2)] \RG{player = keeper}
\item[(3)] \RG{visible[P_1]=1;}\\\$ $t_{1,s'}$;\\\RG{visible[P_1]=0;}
\item[] \ldots
\item[] \RG{visible[P_n]=1;}\\\$ $t_{n,s'}$;\\\RG{visible[P_n]=0;}
\item[(4)] \RG{player=P_$\iota(s')$}\quad if $\iota(s') \neq \text{Nature}$, and\\\RG{player=random}\quad if $\iota(s') = \text{Nature}.$
\end{enumerate}

\bigskip

Therefore, the player $P_r$ in its move chooses $t_{\iota(s),s'}$, which determines its next information set.
The other tags are revealed by the keeper to the other players (including $P_r$, though it is not necessary), depending on the actual next game state $s'$.

To handle randomness, for a state $s$ with $\iota(s) = \text{Nature}$, we add paths from $q_s$ to the children $q_{s'}$ as above, but with the following modification:
Instead of (1), we add multiple edges, each with a distinct tag, to obtain the required probability distribution on the moves.
By duplicating edges, we can obtain any rational probability distribution.
These tags are not seen by the players.
\end{proof}

The \emph{type length} is the number of arrows plus one in the type definition.
Set types have type length $1$.

\begin{theorem}
Given a game description in RG, the problem of deciding whether there is a legal move initial game state is EXPSPACE-complete.
The same holds for verifying conditions (1)--(5) of a proper game description.
If the maximum type length is fixed, then these proble ms become PSPACE-complete.
\end{theorem}
\begin{proof}
\noindent$\bullet$ \emph{Legal move solution in EXPSPACE}:

Let $n$ be the length of the description (the number of characters).

Suppose that a variable has a type length $k$:
\[ T_1 \to T_2 \ldots \to T_k, \]
where $T_i$ are set types.
Then the number of symbols that can be stored by this variable equals $|T_1| \cdot \cdots \cdot |T_{k-1}|$ (i.e., one symbol from $T_k$ for each combination of symbols from types $T_1,\ldots,T_{k-1}$).
We need to multiply this by $\log n$, which is the space necessary to store one symbol, but this cost will be negligible.

Now, observe that $n$ is a (certainly much inflated) upper bound on the number of symbols and also on the maximum type length.
The size of a game state is at most $\mathcal{O}(n^n) = \mathcal{O}(2^{n \log n})$.
Hence, we can store up to an exponential number of game states in EXPSPACE.
The number of possible game states is bounded by $\O(n^{n^n})$, hence we can also implement a logarithmic counter in EXPSPACE.
By Savitch's theorem, NEXPSPACE = EXPSPACE, so we can also use nondeterminism.

The algorithm solving the problem is traversing the game tree starting from the initial game state up to the end of a legal move, by nondeterministically guessing at most $\O(n^{n^n})$ transitions (actions), just counting if we have not fallen into a cycle on game states.
It should be noted that we count only simple transitions, and when computing patterns \RG{? a->b} or \RG{! a->b}, we also traverse and count transitions involved by their computation.

\noindent$\bullet$ \emph{Legal move solution in PSPACE when type length is fixed}:

Let $k$ be the maximum type length.
Then the size of a game state is at most $\mathcal{O}(k^n)$.
Hence, we can store a polynomial number of states in PSPACE and a logarithmic counter counting up to $2^{\mathcal{O}(k^n)}$, which is the maximum number of possible game states.

The algorithm solving the problem is the same and works in NPSPACE = PSPACE.

\noindent$\bullet$ \emph{Legal move EXPSPACE-hardness}:

We reduce from a canonical EXPSPACE-complete problem: whether a given Turning machine of size $N$ (the number of states and also the same number of rules, for simplicity) with a tape of length at most $2^N$ accepts the empty input.
The machine can be either deterministic or nondeterministic (it does not complicate the reduction).
We can assume that the tape cells contain binary values ($0$, $1$), it is initially filled with $0$s, the tape cells are indexed by integers from $0$ to $2^N-1$, and the machine starts at the cell $0$.
Let the states of the machine be denoted by $q_0,\ldots,q_{N-1}$, where $q_0$ is the initial state and $q_{N-1}$ is the final (accepting) state.
A rule of the machine is a quintuple $(q_i,a,q_j,b,D)$, which means that being at the state $q_i$ and reading $a \in \{0,1\}$ from the tape at the current position, the machine can switch to state $q_j$, write $b \in \{0,1\}$ at the current position, and go into the adjacent cell in the direction $D \in \{\mathrm{Left},\mathrm{Right}\}$.

We construct a game description in RG such that the initial player has a legal move if and only if the Turning machine can reach state $q_{N-1}$.

The content of the tape will be stored in game states by a special variable \RG!tape!.
We define the type $\mathit{Bool} = \{0,1\}$ and the variable
\RG!tape: Bool -> ... -> Bool!
where \RG{Bool} occurs $N+1$ times.
Then every cell on the tape can be indexed by the binary representation of its index.
We assume the last digit is the least significant, so we obtain an adjacent cell on the tape by adding or subtracting $1$.

We also define the type \RG!Indices = {0,1,...,N-1}! and the variable
\RG!position: Indices -> Bool!.
which will store the current position on the tape.
We can use \RG!Position! directly to index \RG!tape! by the following expression:\\
\RG!tape [N-1] [Position[N-2]] ... [Position[0]]!

Finally, we need to store the state of the machine.
As the indices of states are from $0$ to $N-1$, we can reuse type \RG!Indices! and add the dedicated variable \RG!state: Indices!.

Now, we implement the rules of the Turning machine as an automaton in RG.

There is just one player \RG!machine!, initially set to the variable \RG!Player!.
Game states store the above defined variables \RG!tape!, \RG!position!, \RG!state!, and an auxiliary variable \RG!index: Indices!.

The following is an example code in HRG for $N=4$ implementing the above concepts (which translates polynomially to RG):
\begin{lstlisting}[style=HRG]
domain Player = machine
domain Score = 0 // Irrelevant

domain Bool = 0 | 1
domain Indices = I(X) where X in 0..3
domain IndicesOrNan = Indices | nan

player: PlayerOrSystem = machine

tape: Bool -> Bool -> Bool -> Bool -> Bool = {:{:{:{:0}}}}
position: Indices -> Bool = {:0}
index: Indices = I(0)
state: Indices = I(0)
\end{lstlisting}

We implement going to the adjacent cell on the right by decrementing \RG!Indices! in the binary system.
This uses just a loop and the auxiliary constant map \RG!inc! for incrementing \RG!index!.
Analogously, we implement going to the adjacent cell on the left by incrementing \RG!Indices!.

The following is an HRG fragment implementing these two gadgets for $N=4$:

\begin{lstlisting}[style=HRG]
inc: Indices -> IndicesOrNan
inc(I(X)) = if X < 3 then I(X+1) else nan

graph goRight() {
   index = I(0)
   loop {
     branch {
       check(position[index] == 1)
       position[index] = 0
       check(index != I(3)) // We are on the last cell, so goRight is illegal
       index = inc[index]
     } or {
       check(position[index] == 0)
       position[index] = 1
       return()
     }
  }
}

graph goLeft() {
  index = I(0)
  loop {
    branch {
      check(position[index] == 0)
      position[index] = 1
      check(index != I(3)) // We are on the first cell, so goLeft is illegal
      index = inc[index]
    } or {
      check(position[index] == 1)
      position[index] = 0
      return()
    }
  }
}
\end{lstlisting}

It remains to define the core rules.
There is a central node from which there are $N$ outgoing edges, one for each rule of the machine.

For a rule $(q_i,a,q_j,b,D)$, the branch first checks the state \RG!state == i!, then checks the symbol on the tape at the current position \RG!tape [Position[N-1]] [Position[N-2]] ... [Position[0]] == a!.
Then, it sets the state of the machine $\mathit{state} = j$, modifies the symbol on the tape: \RG!tape [Position[N-1]] [Position[N-2]] ... [Position[0]] = b!, and goes in either left or right by one of the gadgets above.
Applying a rule is finalized by a tag, which ensures that each move has an unique application.
After that, it returns to the central node, unless the next state is $q_{N-1}$, in which case it ends the move by \RG!player = keeper! and goes to \RG!end!.

The following is a scheme of this construction, where $\ldots$ denotes places that depend on $N$ or the rules.
\begin{lstlisting}[style=HRG]
graph rules() {
  loop {
    branch {// Rule 1
      check(state == I(...))
      check(tape[position[I(N-1)]]...[position[I(0)]] == ...)
      state = I(...)
      tape[position[I(N-1)]]...[position[I(0)]] = ...
      ... // goLeft() or goRight()
      DOLLAR R1
    } or {
      ...
    } or {// Accepting
      ...
      state = I(N)
      ...
      DOLLAR RN
      player = keeper
      end()
    }
  }
}
\end{lstlisting}

In the code repository, \texttt{turingMachine.hrg} implements a simple Turing machine for $N=4$ following the construction from the proof.

\noindent$\bullet$ \emph{Legal move PSPACE-hardness when type length is 1 (no arrows)}:\\

Instead of one variable for the tape, we create $N$ variables \RG!tape0!, $\ldots$, \RG!tape$N$-1! of type \RG!Bool!.
For reading and writing to the current position indicated by \RG!position: Indices!, we create a branch with $N$ cases for each cell.
The resulting code is still polynomial in $N$.

\noindent$\bullet$ \emph{Verifying problems solution in EXPSPACE}:\\

For conditions~(1),~(3),~(4), we nondeterministically guess how to reach a game state for which the condition does not hold, and then verify it in a straightforward way.

For condition~(2), we do the same, but verifying requires guessing a move (we guess tags one by one) and simultaneously reach two next game states.

For condition~(5), we use a counter in exponential space, just as for the existence of a legal move. Overflowing the counter means that a game state has been repeated, thus the game is not finite.

\noindent$\bullet$ \emph{Verifying problems solution in PSPACE when the type length is fixed}:\\

It works in the same way as above, just fitting in polynomial space because game states have polynomial size.

\noindent$\bullet$ \emph{Verifying problems EXPSPACE-hardness}:\\

For each condition, it is trivial to construct a part of the automaton with a node such that the condition is violated if we reach a game state with that node.

We first reduce to the problem of the existence of a legal move.
Then we modify it as follows.
At the node after that move, we append a part violating the condition.
Also, we add an edge that ends a trivial move and ends the game, ensuring that a legal move from the beginning always exists.
Hence, the game description is proper if and only if there does not exist a legal move for the given description.
\end{proof}


\newpage
\section{Optimizations}

Transformations can be categorized into five types:

\begin{enumerate}
    \item \textbf{Expression-oriented:} These optimizations simplify the automaton by propagating constants, inlining variable assignments, merging nested access expressions, and compacting a series of comparisons.
    \item \textbf{Joins:} These transformations detect local patterns in the automaton, such as forks with common or exclusive actions.
    \item \textbf{Prunings:} These optimizations remove unused variables, constants, and unreachable edges.
    \item \textbf{Reachability-oriented:} The goal of these transformations is to simplify and inline reachability checks when it is legal.
    \item \textbf{Normalizations:} These transformations do not necessarily optimize the game. Examples include adding explicit casts in expressions and constant normalization.
\end{enumerate}

\subsection{Expression-oriented Optimizations}

\subsubsection{Compact comparisons} 
Optimizes selective comparisons with negations if it would result in a smaller automaton.

\begin{figure}[!h]
  \centering
  \begin{minipage}[b]{0.4\textwidth}
    \begin{lstlisting}[style=RG, gobble=8]
        type A = { 1, 2, 3 };
        var x: A = 1;
        begin, end: x == 1;
        begin, end: x == 2;
    \end{lstlisting}
    \caption{Before compact comparisons}
  \end{minipage}
  \begin{minipage}[b]{0.4\textwidth}
    \begin{lstlisting}[style=RG, gobble=8]
        type A = { 1, 2, 3 };
        var x: A = 1;
        begin, end: x != 3;
    \end{lstlisting}
    \caption{After compact comparisons}
  \end{minipage}
\end{figure}

\subsubsection{Inline assignment}
For each assignment \RG!x=expr!, collect all references to \RG!x! until it is reassigned. If the value of all variables used in \RG!expr! is the same at the assignment point and at every usage of \RG!x!, remove this assignment and replace each usage of \RG!x! with \RG!expr!.

\begin{figure}[!h]
  \centering
  \begin{minipage}[b]{0.4\textwidth}
    \begin{lstlisting}[style=RG, gobble=8]
        begin, a: coordX = board[x];
        a, b: coordX != coordY;
        b, c: coordX = up[coordX];
    \end{lstlisting}
    \caption{Before inline assignment}
  \end{minipage}
  \begin{minipage}[b]{0.4\textwidth}
    \begin{lstlisting}[style=RG, gobble=8]
        begin, a: ;
        a, b: board[x] != coordY;
        b, c: coordX = up[board[x]];
    \end{lstlisting}
    \caption{After inline assignment}
  \end{minipage}
\end{figure}

\subsubsection{Merge accesses}
For each expression \RG!const1[const2[UNDERSCORE]]! where \RG!const1! is a constant of type \RG!B->C! and \RG!const2! is a constant of type \RG!A->C!, create new constant \RG!const1UNDERSCOREconst2! of type \RG!A->C! and replace with it all usages of \RG!const1[const2[UNDERSCORE]]!.

\begin{figure}[!h]
  \centering
  \begin{minipage}[b]{0.4\textwidth}
    \begin{lstlisting}[style=RG, gobble=8]
        const MapAB: A -> B = { 1: 1, :2 };
        const MapBC: B -> C = { 1: 2, 2: 3, :4 };
        begin, end: x = MapBC[MapAB[1]];
    \end{lstlisting}
    \caption{Before merge accesses}
  \end{minipage}
  \begin{minipage}[b]{0.4\textwidth}
    \begin{lstlisting}[style=RG, gobble=8]
        const MapBC_MapAB: A -> C = 
            { 1: 2, :3 };
        begin, end: x = MapBC_MapAB[1];
    \end{lstlisting}
    \caption{After merge accesses}
  \end{minipage}
\end{figure}

\subsubsection{Propagate constants}
Inlines constants and variables of known, constant value.

\begin{figure}[!h]
  \centering
  \begin{minipage}[b]{0.4\textwidth}
    \begin{lstlisting}[style=RG, gobble=8]
        type A = { 1, 2, 3, 4 };
        const down: A -> A = 
            { 4:3, 3:2, :1 };
        var x: A = 3;
        var y: A = 1;
        begin, a: y = A(*);
        a, b: y == down[x];
    \end{lstlisting}
    \caption{Before propagate constants}
  \end{minipage}
  \begin{minipage}[b]{0.4\textwidth}
    \begin{lstlisting}[style=RG, gobble=8]
        type A = { 1, 2, 3, 4 };
        const down: A -> A = 
            { 4:3, 3:2, :1 };
        var x: A = 3;
        var y: A = 1;
        begin, a: y = A(*);
        a, b: y == 2;
    \end{lstlisting}
    \caption{After propagate constants}
  \end{minipage}
\end{figure}

\subsubsection{Reorder conditions}
This optimization is unsafe because it can change the semantics of the game.
If a certain edge label appears on multiple paths starting at a given node, move edges with this label to just after that node.

\begin{figure}[!h]
  \centering
  \begin{minipage}[b]{0.4\textwidth}
    \begin{lstlisting}[style=RG, gobble=8]
        begin, a: 4 == 4;
        begin, b: 2 == 2;
        a, a1: 1 == 1;
        b, b1: 4 == 4;
    \end{lstlisting}
    \caption{Before reorder conditions}
  \end{minipage}
  \begin{minipage}[b]{0.4\textwidth}
    \begin{lstlisting}[style=RG, gobble=8]
        begin, a: 4 == 4;
        begin, b: 4 == 4;
        a, a1: 1 == 1;
        b, b1: 2 == 2;
    \end{lstlisting}
    \caption{After reorder conditions}
  \end{minipage}
\end{figure}

\subsubsection{Skip self assignments and comparisons}
Replaces self assignments \RG!x=x! and self comparisons \RG!x==x! with skip edges.

\begin{figure}[!h]
  \centering
  \begin{minipage}[b]{0.4\textwidth}
    \begin{lstlisting}[style=RG, gobble=8]
        begin, a1: x == x;
        begin, a2: x != x;
        begin, a3: x = x;
    \end{lstlisting}
    \caption{Before skip self assignments and comparisons}
  \end{minipage}
  \begin{minipage}[b]{0.4\textwidth}
    \begin{lstlisting}[style=RG, gobble=8]
        begin, a1: ;
        begin, a3: ;
    \end{lstlisting}
    \caption{After skip self assignments and comparisons}
  \end{minipage}
\end{figure}

\subsection{Joins}

\subsubsection{Join exclusive edges}
Joins multiedges with exclusive labels.

\begin{figure}[!h]
  \centering
  \begin{minipage}[b]{0.4\textwidth}
    \begin{lstlisting}[style=RG, gobble=8]
        begin, end: ? a -> b;
        begin, end: ! a -> b;
        c, d: x == y;
        c, d: x != y;
    \end{lstlisting}
    \caption{Before join exclusive edges}
  \end{minipage}
  \begin{minipage}[b]{0.4\textwidth}
    \begin{lstlisting}[style=RG, gobble=8]
        begin, end: ;
        c, d: ;
    \end{lstlisting}
    \caption{After join exclusive edges}
  \end{minipage}
\end{figure}

\subsubsection{Join fork prefixes}
Joins paths with identical labels from the same node.

\begin{figure}[!h]
  \centering
  \begin{minipage}[b]{0.4\textwidth}
    \begin{lstlisting}[style=RG, gobble=8]
        begin, b: 1 == 1;
        begin, c: 1 == 1;
        b, end: 2 == 2;
        c, d: ;
    \end{lstlisting}
    \caption{Before join fork prefixes}
  \end{minipage}
  \begin{minipage}[b]{0.4\textwidth}
    \begin{lstlisting}[style=RG, gobble=8]
        begin, b: 1 == 1;
        b, end: 2 == 2;
        c, d: ;
        b, c: ;
    \end{lstlisting}
    \caption{After join fork prefixes}
  \end{minipage}
\end{figure}

\subsubsection{Join fork suffixes}
Joins paths with identical labels leading to the same node.

\begin{figure}[!h]
  \centering
  \begin{minipage}[b]{0.4\textwidth}
    \begin{lstlisting}[style=RG, gobble=8]
        begin, a1: x == 1;
        begin, a2: x == 2;
        a1, end: 0 == 0;
        a2, end: 0 == 0;
    \end{lstlisting}
    \caption{Before join exclusive edges}
  \end{minipage}
  \begin{minipage}[b]{0.4\textwidth}
    \begin{lstlisting}[style=RG, gobble=8]
        begin, a1: x == 1;
        begin, a1: x == 2;
        a1, end: 0 == 0;
    \end{lstlisting}
    \caption{After join exclusive edges}
  \end{minipage}
\end{figure}

\subsection{Prunnings}

\subsubsection{Compact skip edges}
Removes skip edges when possible.

\begin{figure}[!h]
  \centering
  \begin{minipage}[b]{0.4\textwidth}
    \begin{lstlisting}[style=RG, gobble=8]
        begin, b: 1 == 1; 
        b, c: ; 
        c, end: 2 == 2;
    \end{lstlisting}
    \caption{Before compact skip edges}
  \end{minipage}
  \begin{minipage}[b]{0.4\textwidth}
    \begin{lstlisting}[style=RG, gobble=8]
        begin, c: 1 == 1; 
        c, end: 2 == 2;
    \end{lstlisting}
    \caption{After compact skip edges}
  \end{minipage}
\end{figure}

\subsubsection{Prune unreachable nodes}
Removes edges and nodes that are not reachable from the \RG!begin! node or are not on the path to either \RG!end! node or a reachability target.

\begin{figure}[!h]
  \centering
  \begin{minipage}[b]{0.4\textwidth}
    \begin{lstlisting}[style=RG, gobble=8]
        begin, b: ? r1 -> target;
        b, end: ;
        a, end: ;
        r1, r2: ;
        r1, target: ;
    \end{lstlisting}
    \caption{Before prune unreachable nodes}
  \end{minipage}
  \begin{minipage}[b]{0.4\textwidth}
    \begin{lstlisting}[style=RG, gobble=8]
        begin, b: ? r1 -> target;
        b, end: ;
        r1, target: ;
    \end{lstlisting}
    \caption{After prune unreachable nodes}
  \end{minipage}
\end{figure}

\subsubsection{Prune unused constants and variables}
Removes unused constants and variables.

\subsubsection{Skip artificial tags}
Skips edges with tags marked with the $@artificialTag <tag>$ pragma.
This optimization runs at the very end, so that no further transformations take place after artificial tags are removed.

\begin{figure}[!h]
  \centering
  \begin{minipage}[b]{0.4\textwidth}
    \begin{lstlisting}[style=RG, gobble=8]
        begin, end: DOLLAR t;
        @aritificialTag t;
    \end{lstlisting}
    \caption{Before skip artificial tags}
  \end{minipage}
  \begin{minipage}[b]{0.4\textwidth}
    \begin{lstlisting}[style=RG, gobble=8]
        begin, end: ;
        @aritificialTag t;
    \end{lstlisting}
    \caption{After skip artificial tags}
  \end{minipage}
\end{figure}

\subsubsection{Skip redundant tags}
Skips edges with tags that are not needed when choosing a move - they appear in every possible choice at the same position.

\begin{figure}[!h]
  \centering
  \begin{minipage}[b]{0.4\textwidth}
    \begin{lstlisting}[style=RG, gobble=8]
        begin, a: DOLLAR t1.
        begin, b: DOLLAR t2;
        begin, c: DOLLAR t3;
        a, end: DOLLAR rt;
        b, end: DOLLAR rt;
        c, end: DOLLAR rt;
    \end{lstlisting}
    \caption{Before skip redundant tags}
  \end{minipage}
  \begin{minipage}[b]{0.4\textwidth}
    \begin{lstlisting}[style=RG, gobble=8]
        begin, a: DOLLAR t1.
        begin, b: DOLLAR t2;
        begin, c: DOLLAR t3;
        a, end: ;
        b, end: ;
        c, end: ;
    \end{lstlisting}
    \caption{After skip redundant tags}
  \end{minipage}
\end{figure}

\subsection{Reachability-oriented Optimizations}

\subsubsection{Compact reachability }
Moves leading and trailing non-conditional edges out of reachability checks.

\begin{figure}[!h]
  \centering
  \begin{minipage}[b]{0.4\textwidth}
    \begin{lstlisting}[style=RG, gobble=8]
        a, b: ? x -> y;
        x0, x1: ;
        x1, x2: coordX == coordY;
        x2, y0: ;
    \end{lstlisting}
    \caption{Before compact reachability}
  \end{minipage}
  \begin{minipage}[b]{0.4\textwidth}
    \begin{lstlisting}[style=RG, gobble=8]
        a, b: ? x1 -> x2;
        x0, x1: ;
        x1, x2: coordX == coordY;
        x2, y0: ;
    \end{lstlisting}
    \caption{After compact reachability}
  \end{minipage}
\end{figure}

\subsubsection{Inline reachability}
Inlines reachability subautomaton in place of reachability check. 
All variables changed inside the subautomaton must be reassigned before being read.

\begin{figure}[!h]
  \centering
  \begin{minipage}[b]{0.4\textwidth}
    \begin{lstlisting}[style=RG, gobble=8]
        begin, end: ? a -> c;
        a, b: 1 == 1;
        b, c: 2 == 2;
    \end{lstlisting}
    \caption{Before inline reachability}
  \end{minipage}
  \begin{minipage}[b]{0.4\textwidth}
    \begin{lstlisting}[style=RG, gobble=8]
        begin, a: ;
        a, b: 1 == 1;
        b, end: 2 == 2;
    \end{lstlisting}
    \caption{After inline reachability}
  \end{minipage}
\end{figure}

\subsubsection{Skip unused tags}
Removes tags inside reachability subautomatons.

\begin{figure}[!h]
  \centering
  \begin{minipage}[b]{0.4\textwidth}
    \begin{lstlisting}[style=RG, gobble=8] 
        begin, end: ? t1 -> t3;
        t1, t2: DOLLAR 1;
        t2, t3: DOLLAR 2;
    \end{lstlisting}
    \caption{Before skip unused tags}
  \end{minipage}
  \begin{minipage}[b]{0.4\textwidth}
    \begin{lstlisting}[style=RG, gobble=8]
        begin, end: ? t1 -> t3;
        t1, t2: ;
        t2, t3: ;
    \end{lstlisting}
    \caption{After skip unused tags}
  \end{minipage}
\end{figure}

\subsection{Other Transformations}

\subsubsection{Add explicit casts}
Adds type casts to every expression.

\begin{figure}[!h]
  \centering
  \begin{minipage}[b]{0.4\textwidth}
    \begin{lstlisting}[style=RG, gobble=8] 
        type T = { 1 };
        var t: T = 1;
        x1, y1: t == t;
    \end{lstlisting}
    \caption{Before add explicit casts}
  \end{minipage}
  \begin{minipage}[b]{0.4\textwidth}
    \begin{lstlisting}[style=RG, gobble=8]
        type T = { 1 };
        var t: T = 1;
        x1, y1: T(t) == T(t);
    \end{lstlisting}
    \caption{After add explicit casts}
  \end{minipage}
\end{figure}

\subsubsection{Expand any assignment}
For each edge with assignment \RG!x = A(*)! and for each member of type \RG!A!, creates a new edge assigning that member to \RG!x!.

\begin{figure}[!h]
  \centering
  \begin{minipage}[b]{0.4\textwidth}
    \begin{lstlisting}[style=RG, gobble=8]
        type Coord = { 0, 1, 2 };
        a, b: coordX = Coord(*); 
        b, c: coordY = Coord(*);
    \end{lstlisting}
    \caption{Before expand any assignment}
  \end{minipage}
  \begin{minipage}[b]{0.4\textwidth}
    \begin{lstlisting}[style=RG, gobble=8]
        type Coord = { 0, 1, 2};
        a, b: coordX = 0;
        a, b: coordX = 1;
        a, b: coordX = 2;
        b, c: coordY = 0;
        b, c: coordY = 1;
        b, c: coordY = 2;
    \end{lstlisting}
    \caption{After expand any assignment}
  \end{minipage}
\end{figure}

\subsubsection{Expand variable tags}
For each edge with a variable tag \RG!DOLLAR DOLLAR x!, where \RG!x! is a variable of type !A!, and for each member \RG!a! of type \RG!A!, creates a new edge with tag \RG!DOLLAR a!, preceded with comparison \RG!x == A(a)!.

\begin{figure}[!h]
  \centering
  \begin{minipage}[b]{0.4\textwidth}
    \begin{lstlisting}[style=RG, gobble=8]
        type Coord = { 0, 1, 2 };
        a, b: DOLLAR DOLLAR coord;
    \end{lstlisting}
    \caption{Before expand variable tag}
  \end{minipage}
  \begin{minipage}[b]{0.4\textwidth}
    \begin{lstlisting}[style=RG, gobble=8]
        type Coord = { 0, 1, 2 };
        a, b0: coord == Coord(0);
        b0, b: DOLLAR 0;
        a, b1: coord == Coord(1);
        b1, b: DOLLAR 1;
        a, b2: coord == Coord(2);
        b2, b: DOLLAR 2;
    \end{lstlisting}
    \caption{After expand variable tag}
  \end{minipage}
\end{figure}

\subsubsection{Mangle symbols}
Replaces symbol names with new values.

\begin{figure}[!h]
  \centering
  \begin{minipage}[b]{0.4\textwidth}
    \begin{lstlisting}[style=RG, gobble=8]
        begin,end: playerTurn = Player(X);
        turn,move: ? move -> set;
    \end{lstlisting}
    \caption{Before mangle symbols}
  \end{minipage}
  \begin{minipage}[b]{0.4\textwidth}
    \begin{lstlisting}[style=RG, gobble=8]
        begin, end: _t = Player(X);
        _w, _u: ? _u -> _v;
    \end{lstlisting}
    \caption{After mangle symbols}
  \end{minipage}
\end{figure}

\subsubsection{Normalize constants}
Make maps appear only at the top level.

\begin{figure}[!h]
  \centering
  \begin{minipage}[b]{0.4\textwidth}
    \begin{lstlisting}[style=RG, gobble=8]
        var X: Bool -> Bool -> Bool
            = { :{ :0 } };
    \end{lstlisting}
    \caption{Before normalize constants}
  \end{minipage}
  \begin{minipage}[b]{0.4\textwidth}
    \begin{lstlisting}[style=RG, gobble=8]
        const HoistedUNDERSCORE1: Bool -> Bool
            = { :0 };
        const HoistedUNDERSCORE2: Bool -> Bool -> Bool
            = { :HoistedUNDERSCORE1 };
        var X: Bool -> Bool -> Bool
            = HoistedUNDERSCORE2;
    \end{lstlisting}
    \caption{After normalize constants}
  \end{minipage}
\end{figure}

\subsubsection{Normalize types}
Make arrow types appear only in the definition and be at most one level deep.

\begin{figure}[!h]
  \centering
  \begin{minipage}[b]{0.4\textwidth}
    \begin{lstlisting}[style=RG, gobble=8]
        type X = { x } -> { y };
    \end{lstlisting}
    \caption{Before normalize types}
  \end{minipage}
  \begin{minipage}[b]{0.4\textwidth}
    \begin{lstlisting}[style=RG, gobble=8]
        type X = Type1 -> Type2;
        type Type1 = { x };
        type Type2 = { y };
    \end{lstlisting}
    \caption{After normalize types}
  \end{minipage}
\end{figure}

\subsection{Data--flow Analysis}

The analysis uses an iterative worklist algorithm and is generalized for multiple use cases.

\begin{lstlisting}
pub trait Analysis {
    type Domain: Clone + PartialEq;

    fn bot() -> Self::Domain;
    fn extreme(program: &Game<Id>) -> Self::Domain;
    fn gen(input: Self::Domain, edge: &Arc<Edge<Id>>) -> Self::Domain
    fn join(a: Self::Domain, b: Self::Domain) -> Self::Domain;
    fn kill(input: Self::Domain, edge: &Arc<Edge<Id>>) -> Self::Domain

    fn transfer(input: Self::Domain, edge: &Arc<Edge<Id>>) -> Self::Domain {
        Self::gen(Self::kill(input, edge), edge)
    }
}
\end{lstlisting}

The \Rust!type Domain! defines the type of knowledge. 
\Rust!extreme! and \Rust!bot! represent initial knowledge respectively in the \Rust!begin! node and in every other node. 
Functions \Rust!gen! and \Rust!kill! define how the knowledge changes when passing an edge, as seen in \Rust!transfer!.

\subsection{Optimizations results}

\begin{figure}[h]
    \centering
    \includegraphics[width=.9\textwidth]{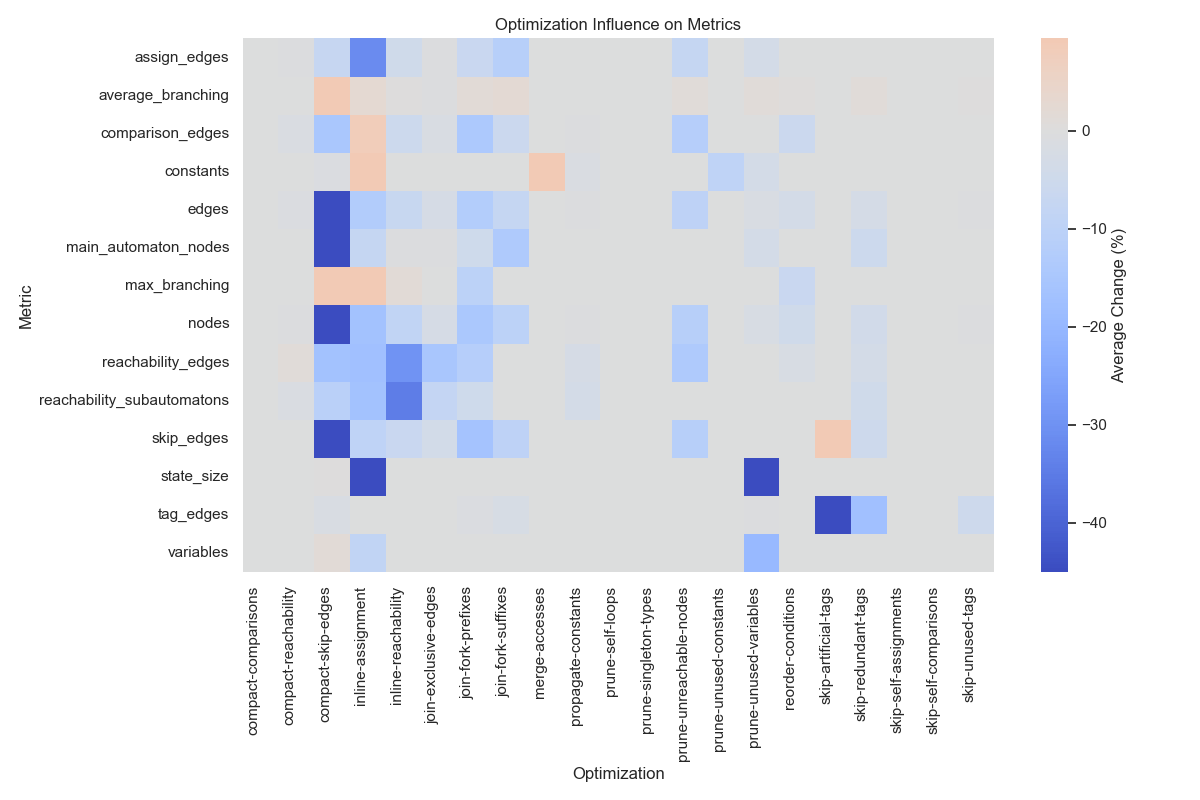}
    \caption{Effect of optimizations on various metrics}
    \label{fig:optimizations_metrics}
\end{figure}

\begin{figure}
    \centering
    \vspace{1cm}
    \includegraphics[width=.9\textwidth]{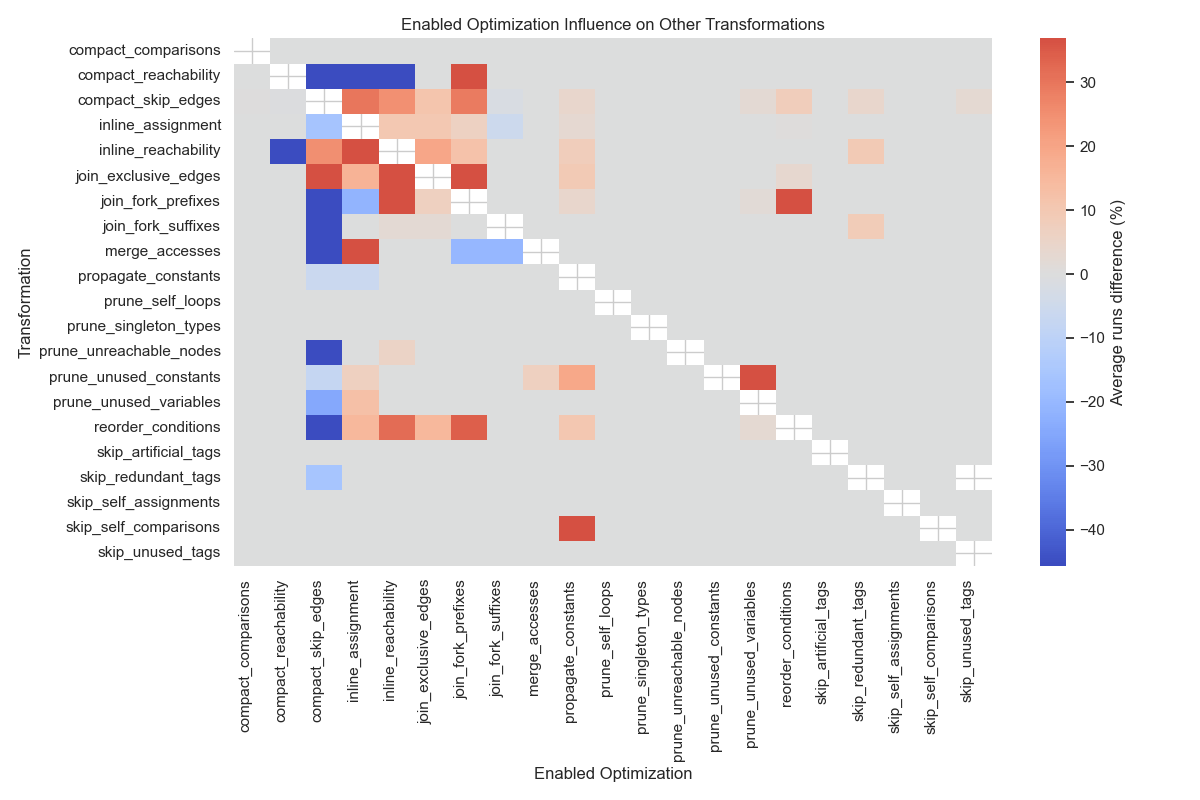}
    \caption{How one optimization enables another}
    \label{fig:optimizations_enabling}
    \vspace{1cm}
\end{figure}



\newpage
\section{Compiler}
The compiler creates files containing implementation of reasoner in C++ based on the AST.
The following section presents details of the implementation and on what basis it is generated.

\subsection{Types of Functions in Reasoner}
Generated forward model consists of a set of functions that allow for the representation of the game state and its modification.
The most important of these functions are:

\subsubsection{Checking terminal state}
A play in every game is always finished after reaching \RG{end} node of automaton, so verifying if terminal state is reached is trivial - numeric id of current node is compared with numeric id of \RG{end}.
The representation of node in the *.rg, *.hrg is a string, and it is changed to a number in the compiler.
    \begin{lstlisting}[language=c++, gobble=4]
    bool GameState::isTerminal()
    \end{lstlisting}

\subsubsection{Getting outcome for players}
This functionality should be used only when the terminal state is reached.
It simply reads a value from the special variable \RG!goals:Goals!.
    \begin{lstlisting}[language=c++, gobble=4]
    Score GameState::getPlayerScore(Player player)
    \end{lstlisting}

\subsubsection{Getting current player control}
Current player can be any of the players defined in implemented game or any of the special ones: \RG!keeper! or \RG!random!.
    \begin{lstlisting}[language=c++, gobble=4]
    PlayerOrSystem GameState::getCurrentPlayer()
    \end{lstlisting}

\subsubsection{Calculating legal moves}
Generating legal moves is based on the search of automaton by applying valid actions and collecting each unique labeling from every possible move walk.
This function starts the search and puts all possible moves from the current position into the vector:
    \begin{lstlisting}[language=c++, gobble=4]
    void GameState::getAllMoves(std::vector<Move>& moves, RgCache& rgCache)
    {
        rgCache.reset();
        moves.clear();
        Move mr;
        switch (currentState)
        {
            case 46:
            {
                state_begin(moves, mr.mr);
                return;
            }
            case 49:
            {
                state_move_forward(moves, mr.mr);
                return;
            }
        }
    }
    \end{lstlisting}

In order to reduce size of code, is the switch statement the only accepted node IDs are those nodes, before which a player change occurs and the node \RG{begin} from which the game starts, because these are the only nodes from which move application might take place.

Each node $q \in Q$ of automaton has its own function \emph{state\_q} which contains the execution of actions on the edges outgoing from this node and calling functions for successor nodes.
The tags encountered during the automaton's transition are saved to the \emph{mr} (move representation) container, and then when a player assignment action is encountered, the move is added to the set of possible moves stored in the \emph{moves} variable.

The search algorithm needs to recognize if a state of the game has already been visited.
This information is stored inside \emph{RGCache} structure which might be implemented differently for various games.

\subsubsection{Move application}
The move application is implemented as finding a move walk corresponding to the given move and applying all transitions from this walk.
    \begin{lstlisting}[language=c++, gobble=4]
    void GameState::applyMove(const Move& mr, RgCache& rgCache)
    {
        rgCache.reset();
        switch (currentState)
        {
            case 46:
            {
                apply_state_begin(mr.mr, rgCache);
                return;
            }
            case 49:
            {
                apply_state_move_forward(mr.mr, rgCache);
                return;
            }
        }
    }
    \end{lstlisting}

Function \emph{apply\_state\_X} analogous to \emph{state\_X} -- but instead of considering all types of actions, we are only interested in the assignment and tag actions.
In the case of the tag action, it is compared with the tag from the movement container.
If the comparison fails, the function ends with the false result and the search continues through other edges.

\subsubsection{Move application for \RG!keeper!}
There might be a few possible paths for the \RG!keeper! to the next node where the \RG{player} is changed.
However, all of them must be equivalent and game state after \RG{keeper's} move is always the same.
Finding the only valid move walk is implemented in a separate function \emph{applyAnyMove}, which can be used only for \RG!keeper!.
Such approach works faster than generating and then applying a specific move.
    \begin{lstlisting}[language=c++, gobble=4]
    bool GameState::applyAnyMove(RgCache& rgCache)
    \end{lstlisting}

\subsubsection{Reachability check}
The reachability action \RG{?X->Y} requires checking whether it is possible to walk from node \RG{X} to node \RG{Y} with the current game semistate.
Functions \emph{is\_legal\_X\_Y\_Z} are created for each expression \RG{?X->Y} or \RG{!X->Y} and each node \RG{Z} which might be visited on the walk from source to destination node.
These functions are not created based on the whole graph, but only a subset of the graph edges that could be used to walk from vertex \RG{X} to \RG{Y}.
The signature \emph{is\_legal\_X\_Y\_Z} denotes the state function for node \RG{Z} in subgraph \RG{X->Y}.
Figure~\ref{fig:full-graph} shows the entire graph, while figure~\ref{fig:reduced-graph} shows the reduced graph for the expression \RG{1->7}

\begin{figure}[!h]
  \centering
  \begin{minipage}[b]{0.4\textwidth}
    \includegraphics[width=\textwidth]{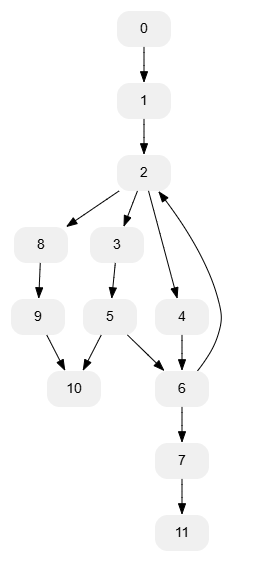}
    \caption{The whole graph}
    \label{fig:full-graph}
  \end{minipage}
  \hfill
  \begin{minipage}[b]{0.3\textwidth}
    \includegraphics[width=\textwidth]{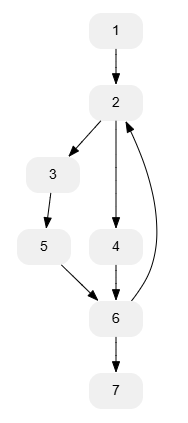}
    \caption{Reduced graph for pattern \RG{1->7}}
    \label{fig:reduced-graph}
  \end{minipage}
\end{figure}

    \begin{lstlisting}[language=c++, gobble=4]
    bool GameState::is_legal_checkline_endcheckline_checkline(RgCache& rgCache)
    \end{lstlisting}

\subsection{Handling Edge Actions}
\subsubsection{Assignment}
The assignment is handled as a normal assignment in C++.
Before the assignment operation is performed, a temporary variable is created that stores the old value, so it can be restored after traversing the edge.
    \begin{lstlisting}[language=c++, gobble=4]
    // action: position = V0;
    const auto tmp1 = position;
    position = V0;
    state_choose_position(moves, mr, rgCache);  // continue search in next node
    position = tmp1;
    \end{lstlisting}

Assigning value to variable \RG{player} is handled differently, which means that a complete move walk was found and current labeling is inserted to the vector of all moves.

\subsubsection{Any assignment}
Assigning any value is handled in a similar way as standard assignment.
The only difference is that such instruction is placed in a loop over all possible values of given type.

    \begin{lstlisting}[language=c++, gobble=4]
    // action: pos = Position(*);
    const auto tmp1 = pos;
    for (int posIt = v1; posIt <= v64; ++posIt)
    {
        pos = posIt;
        state_set_position(moves, mr, rgCache);  // continue search in next node
    }
    pos = tmp1;
    \end{lstlisting}

\subsubsection{Comparison}
Comparisons are handled in \emph{if} statement.
    \begin{lstlisting}[language=c++, gobble=4]
    // action: board[position] == whitePawn
    if (board[position] == whitePawn)
    {
        state_capture(moves, mr, rgCache)  // continue search in next node
    }
    \end{lstlisting}

\subsubsection{Tag}
Numeric id of tag is inserted to container representing a move.
    \begin{lstlisting}[language=c++, gobble=4]
    // action: $left
    void GameState::state_selectDirection(
        std::vector<Move>& moves, move_representation& mr, RgCache& rgCache)
    {
        mr.push_back(1);
    \end{lstlisting}

\subsubsection{Variable tag}
The value of variable with optional padding, needed to preserve uniqueness among values of all tags, is inserted to the container representing a move.
    \begin{lstlisting}[language=c++, gobble=4]
    // action: $position
    void GameState::state_selectPosition(
        std::vector<Move>& moves, move_representation& mr, RgCache& rgCache)
    {
        mr.push_back(position + 3);
    \end{lstlisting}

\subsubsection{Reachability}
Reachability check are implemented in dedicated functions \emph{is\_legal*} described previously and \emph{if} statement
    \begin{lstlisting}[language=c++, gobble=4]
    // action: ? move -> moved
    if (is_legal_move_moved_move(rgCache))
    \end{lstlisting}

\subsection{Pragmas Implementation}

\subsubsection{Pragma \RG!@disjoint! and \RG!@disjointExhaustive!}
            \begin{lstlisting}[language=c++, gobble=12]
                @disjointExhaustive node : nodes+;
                @disjoint node : nodes+;
                {disjoint X : x_1, x_2, ..., x_n}
            \end{lstlisting}
The \RG{Disjoint} pragma says that at most one of the edges between vertex \emph{X} and vertices \emph{x\_i} is valid.
Exhaustive variant says that exactly one is valid.
Given this information, if the first action between vertex \emph{X} and vertices \emph{x\_i} is a comparison or reachability check, then number of evaluated edges might be reduced.
For standard variant search might be stopped finding the first valid transition.
In the exhaustive version, checking action from the last edge can be skipped if the previous checked action fails.

  \begin{lstlisting}[language=c++, gobble=4]
    // Without Disjoint
    void GameState::state_checkwin(
        std::vector<Move>& moves, move_representation& mr, RgCache& rgCache)
    {
        if (is_legal_checkline_endcheckline_checkline(rgCache))
        {
            state_win(moves, mr, rgCache);
        }
        
        if (!is_legal_checkline_endcheckline_checkline(rgCache))
        {
            state_nextturn(moves, mr, rgCache);
        }
    }

    // With Disjoint
    void GameState::state_checkwin(
        std::vector<Move>& moves, move_representation& mr, RgCache& rgCache)
    {
        if (is_legal_checkline_endcheckline_checkline(rgCache))
        {
            state_win(moves, mr, rgCache);
            return;
        }

        if (!is_legal_checkline_endcheckline_checkline(rgCache))
        {
            state_nextturn(moves, mr, rgCache);
            return;
        }
    }

    // With DisjointExhaustive
    void GameState::state_checkwin(
        std::vector<Move>& moves, move_representation& mr, RgCache& rgCache)
    {
        if (is_legal_checkline_endcheckline_checkline(rgCache))
        {
            state_win(moves, mr, rgCache);
            return;
        }
        state_nextturn(moves, mr, rgCache);
    }
    \end{lstlisting}

\subsubsection{Pragma \RG!@tagIndex! and \RG!@tagIndexMax!}
            \begin{lstlisting}[language=c++, gobble=12]
                @tagIndex node+ : index;
                @tagIndexMax node+ : index;
            \end{lstlisting}

Both pragmas affect the type of container in which the set of tags representing the movement is stored.
The \RG{TagIndex} $(node, pos)$ pragma says that a tag on an edge originating from a given node will always appear at position pos in movement container.
\RG{TagIndexMax} pragma says that it can only appear at positions 0 to pos.
If all nodes in the graph have \RG{TagIndex} set, the compiler sets the container type to store movement to \emph{std::array} with size corresponding to the largest value assigned in \RG{TagIndex}.
If \RG{TagIndex} is not set for all nodes but \RG{TagIndexMax} is, the container type is \emph{boost\_container::static\_vector} with size being the maximum value given in \RG{TagIndexMax}.
If both of the above conditions are not met, the type \emph{boost\_container::small\_vector} will be set with static capacity 12.

\subsubsection{Pragma \RG!@simpleApply! and \RG!@simpleApplyExhaustive!}
            \begin{lstlisting}[language=c++, gobble=12]
                {StartingNode EndingNode [list of tags separated by commas - may be empty] 
                list of assignments separated by commas}
                
                @simpleApplyExhaustive rules_1 rules_2 
                    [pos_1: Position, L] 
                    pos = Position(pos_1),
                    player = PlayerOrSystem(keeper);
            \end{lstlisting}
If there is \RG{SimpleApply} pragma in which node X is the starting node then in its \emph{apply\_state\_X} function at the very beginning a call to the \emph{switch\_*} function is added. If \emph{switch\_*} returns true, then \emph{apply\_state\_X} will also return true. If it is \emph{SimpleApplyExhaustive}, the rest of the body of the function \emph{apply\_state\_X} is removed, and the only returned value is that of the function \emph{switch\_*}. The structure of the \emph{switch\_*} function looks like this - if \RG{SimpleApply} from this node had an empty string of tags, it could only be used to change the player, this condition is included in \emph{if (mr.size() == currentMrId)}, where \emph{mr} is a container containing a list of tags representing movement, \emph{currentMrId} is the currently considered position in the container. Then for non-empty tag lists, nested switch statements follow. Additionally in the case of \RG{SimpleApplyExhaustive} the last case condition is transformed to \emph{default}.

            \begin{lstlisting}[language=c++, gobble=12]
            bool GameState::switch_5(const move_representation& mr, RgCache& rgCache)
            {
                // Non empty tag lists
                if (mr.size() > currentMrId)
                {
                    switch (mr[currentMrId++])
                    {
                        case 269:
                        {
                            switch (mr[currentMrId++])
                            {
                                default:
                                    move_horizontal_direction = right;
                                    return apply_state_r2_t2_m27_move(mr,rgCache);
                            }
                        }
                        case 271:
                        {
                            move_horizontal_direction = right;
                            return apply_state_rules_2_turn_2_move(mr,rgCache);
                        }
            
                        default:
                            move_horizontal_direction = left;
                            return apply_state_rules_2_turn_2_move_2(mr,rgCache);
                    }
                }
                // Empty tag list
                if (mr.size() != currentMrId)
                {
                    return false;
                }
                stagnation = S__0;
                player = keeper;
                currentState = 1226;
                return true;
            }
            \end{lstlisting}

\subsubsection{Pragma \RG!@unique!}
While performing a search on automaton, it is necessary to check if current game state was already visited.
However, sometimes such situation cannot happen for some nodes.
If such nodes are recognized, they are be marked as \RG{unique} and then using cache is not needed in these nodes.

\subsubsection{Pragma \RG!@repeat!}
In general case, checking if some game state was already visited requires a comparison of the game semistate, a sequence of tags visited on the path from the source state to the current automaton state, and the current automaton state.
However, when a list of variables to compare can be reduced to some subset, or even an empty set, this information is encoded in pragma repeat as follows:
\begin{lstlisting}
@repeat node+ : variable+;
\end{lstlisting}
Compiler can use this information to provide a much faster cache for given automaton state implemented as a bitset instead of a hash table, or even a single boolean when the list of variables is empty.
Nevertheless, tracking list of visited tags is still necessary to distinguish paths representing different moves.

When tags are unique, i.e. each tag acion explicitly defines a transition of automaton, it is not necessary to track prefix of move to properly detect if given game state was already visited.
As a result, the cache structure can be optimized even more.
The only difference is that on every transition with tag action, the cache for automaton states listed in pragma repeat must be cleared.
It is still a more effective solution, because cache insertion and lookup operations are faster.

\subsubsection{Pragma \RG!@integer!}
RG supports neither arithmetic nor logical operations.
Such operations must be defined as constants of arrow types.

\begin{lstlisting}[style=HRG]
# HRG example
domain Counter = C(N) where N in 0..100
domain CounterOrNan = nan | Counter
counterIncrease: Counter -> CounterOrNan
counterIncrease(C(N)) = if N == 100 then nan else C(N + 1)
\end{lstlisting}

\begin{lstlisting}[style=RG]
# after translation to RG
@integer 0 : C__0 C__1 ... C__100;
type Counter = { C__0, C__1, ..., C__100 };
type CounterOrNan = { nan, C__0, C__1, ..., C__100 };
const counterIncrease: Counter -> CounterOrNan = { :nan, C__0: C__1, ..., C__99: C__100 };
\end{lstlisting}

Taking into account that symbols of some set type encode integers, it is possible to restore original operations without using constans of map types.
Currently, the following types of arithmetic operations are supported: addition, subtraction, incrementation and decrementation.
Each of them may appear in three variants: overflow, modular or saturation.
Supported logical operations are: less, less or equal, greater, greater or equal.

\subsubsection{Pragma \RG!@iterator!}
As defined before, any assignment operation generates a loop instruction which assigns every possible value from some set type to given variable.
Hovewer, right after such instruction there might be a comparison instruction which reduces the set of possible values.

\begin{lstlisting}[style=RG]
@iterator A B C : coord_temp;
type Coord = { rx0y0, rx1y0, ..., rx7y7 };
const ReachableMap: Coord -> Coord -> Bool = { 
    :{ :0 },
    rx0y1: { :0, rx2y0: 1 },
    rx1y1: { :0, rx3y0: 1 },
    ...,
    rx7y7: { :0, rx6y5: 1, rx5y6: 1 } };

A, B: coord_temp = Coord(*);
B, C: ReachableMap[coord][coord_temp] == 1;
\end{lstlisting}
The path from node \emph{A} to \RG{C} would normally be translated to the following C++ code:

\begin{lstlisting}[language=c++]
for (coord_temp = rx0y1; coord_temp <= rx7y7; ++coord_temp)
{
    if (ReachableMap[coord][coord_temp])
    {
        // instructions outgoing from node C
    }
}
\end{lstlisting}
With the information taken from pragma \RG!@iterator!, a new constant is created which stores iterator ranges for each value of \emph{coord} variable.

\begin{lstlisting}[language=c++]
std::array<std::vector<Coord>, 64> ReachableMapIter = {
    {rx2y0},        // rx0y1: { :0, rx2y0: 1 }
    {rx3y0},        // rx1y1: { :0, rx3y0: 1 }
    ...
    {rx6y5, rx5y6}  // rx7y7: { :0, rx6y5: 1, rx5y6: 1 } }
}

for (auto coord_temp_iter : ReachableMapIter[coord])
{
    coord_temp = coord_temp_iter;
    // no 'if' statement, only instructions outgoing from node C
}
\end{lstlisting}


\newpage
\section{Compatibility with Different Systems}

    \subsection{RBG}
        In RBG, a game description is based on a regular expression and a directed board graph with labelled edges.
        
        The board is encoded in a \RG!type Board=Coord->Piece!, where \RG!Coord! type is a set of all board nodes and \RG!Piece! type is a set of all pieces.
        To encode the transitions, we define a \RG!direction_L:Coord->Coord! for every board edge label \RG!L!.

        There are only three variables: \RG!var board:Board! that represents the current state of the board, \RG!var coord:Coord! that represents the current position on the board, and \RG!var counters:Piece->N! that represents the number of each piece on the board (\RG!N! is a set of integers between 0 and the number of positions).
        Every change on the board is reflected in \RG!counters! (games can use that information).

        To translate the regular expression, we follow Thompson's construction \cite{Thompson1968ProgrammingTechniques}.
        The Kleene star expression results in a loop, which is not optimal for series a of \emph{shifts} (\RG!coord! assignments).
        Instead, we recognize such shift patterns and translate them to more efficient lookups.
        However, some loops are irreducible -- often when shifts are combined with \emph{off}'s (\RG!board! checks).

        All math operations (e.g., adjusting \RG!counters! or operating on custom variables) can overflow.
        RBG's treats overflows as illegal, so all math operations are followed up by a \RG!nan! check.
        Similarly, all shifts may go over the board, which we signalize with and verify with \RG!null!.

        To match the ``no legal moves end the game'' semantic, all moves start with a reachability check, verifying whether the move is valid.
        (A failed check ends the game.)

        The \RG!Player! type matches the list of players in RBG.
        The \RG!Score! type is a set of integers from 0 to the maximum score defined in RBG (usually 100).

    \subsection{GDL}
        In GDL, a game description is a set of Boolean predicates and a game state is a set of facts.
        Translation to RG is based on propositional nets \cite{Schkufza2008Propositional}, which introduces a way of encoding boolean predicates in ter ms of a circuit-based formalism.

        Propositional nets require all predicates to be \emph{grounded}, i.e., without any variables.
        This process may result in an exponential growth of the number of predicates, which makes it infeasible for the most complex games.
        To optimize this process, we prune the static (i.e., always true) and impossible (i.e., always false) predicates during the process.

        Custom predicates are encoded as subautomata and evaluated using reachability checks.
        Negation (\texttt{not}) creates a subautomaton for its content and evaluates it using negated reachability checks.
        All \texttt{true} predicates are encoded as comparisons of the \RG!F_prev! variables with \RG!1! (defined below).

        The set of players and their actions are calculated based on the \texttt{legal} predicates.
        In RG, we define one variable for player's \RG!P! action (\RG!var does_P:does_P_type! initial value is arbitrary).
        Its domain consists of moves possible for that player (different players may have different possible moves).
        
        A player's \RG!P! move consists of one path for every possible action \RG!A!.
        It evaluates the \texttt{legal} predicate using reachability check, sets the \RG!does_P! variable, and tags it (\RG!DOLLARA!).
        Before that, \RG!keeper! sets the \RG!visibility! so that each player sees only their own moves to emulate the simultaneous moves in GDL.

        To transition between the states, we have to evaluate the logical rules defining the set of facts true in the next state.
        In RG, for every fact \RG!F! (i.e., \texttt{base} predicate in GDL) we define two boolean variables, representing its current and next value (\RG!var F_prev:Bool! and \RG!var FUNDERSCOREnext:Bool=0;!).
        The initial values of the current facts are based on the \texttt{init} predicates.
        Once all players select their actions, next values are evaluated using the \texttt{next} predicates (\RG!FUNDERSCOREnext! values are zeroed afterwards to prevent leaking hidden information).

        The \RG!Score! type is a set of all possible \texttt{goals} of the game.
        Once the next state is calculated, the \texttt{terminal} predicate is evaluated.
        When true, final \RG{goals} are calculated and the game ends; otherwise another turn starts.

\newpage
\section{Additional Results}
    All experiments were performed on AMD~Ryzen~9~3950X, 64GB, with Ubuntu 24.04.3 LTS, g++~14.2.0, and GraalVM 25.0.1+8.1 (for Ludii).
    
    \subsection{Translation and Optimization Speed}
    
        \begin{table}[h]
            \centering
            \begin{tabular}{lrrrr}
                \toprule
                Game & No optimizations & All optimizations \\
                \midrule
                    \texttt{alquerque.hrg}                     & 9 ms    & 36 ms \\
                    \texttt{amazons.hrg}                       & 15 ms   & 82 ms \\
                    \texttt{amazons\_split2.hrg}               & 15 ms   & 106 ms \\
                    \texttt{ataxx.hrg}                         & 16 ms   & 29 ms \\
                    \texttt{backgammon.hrg}                    & 90 ms   & 4233 ms \\
                    \texttt{battleships.hrg}                   & 10 ms   & 62 ms \\
                    \texttt{bombardment.hrg}                   & 11 ms   & 14 ms \\
                    \texttt{breakthrough.hrg}                  & 7 ms    & 15 ms \\
                    \texttt{chess.hrg}                         & 39 ms   & 1344 ms \\
                    \texttt{chess\_kingCapture.hrg}            & 23 ms   & 1009 ms \\
                    \texttt{chessTest1.hrg}                    & 16 ms   & 30 ms \\
                    \texttt{clobber.hrg}                       & 7 ms    & 10 ms \\
                    \texttt{connect4.hrg}                      & 5 ms    & 12 ms \\
                    \texttt{diceThrowCompare.hrg}              & 4 ms    & 6 ms \\
                    \texttt{diceThrowGuess.hrg}                & 3 ms    & 6 ms \\
                    \texttt{dotsAndBoxes.hrg}                  & 11 ms   & 17 ms \\
                    \texttt{englishDraughts.hrg}               & 9 ms    & 57 ms \\
                    \texttt{foxAndGeese.hrg}                   & 9 ms    & 7 ms \\
                    \texttt{foxAndGeese\_lud.hrg}              & 14 ms   & 56 ms \\
                    \texttt{gomoku\_freeStyle.hrg}             & 14 ms   & 60 ms \\
                    \texttt{gomoku\_standard.hrg}              & 16 ms   & 78 ms \\
                    \texttt{knightthrough.hrg}                 & 8 ms    & 16 ms \\
                    \texttt{oware.hrg}                         & 17 ms   & 631 ms \\
                    \texttt{pentago.hrg}                       & 35 ms   & 547 ms \\
                    \texttt{pentago\_split.hrg}                & 23 ms   & 535 ms \\
                    \texttt{satSolver.hrg}                     & 4 ms    & 6 ms \\
                    \texttt{shortestPath.hrg}                  & 5 ms    & 6 ms \\
                    \texttt{ticTacDie.hrg}                     & 5 ms    & 12 ms \\
                    \texttt{ticTacToe.hrg}                     & 4 ms    & 11 ms \\
                    \texttt{turingMachine.hrg}                 & 6 ms    & 10 ms \\
                    \texttt{twentyOne.hrg}                     & 18 ms   & 27 ms \\
                    \texttt{ultimateTicTacToe.hrg}             & 7 ms    & 26 ms \\
                \bottomrule
            \end{tabular}
            \caption{Analysis time of all HRG games with 60 seconds time limit.}
        \end{table}
        \newpage
        \begin{table}[h]
            \centering
            \begin{tabular}{lrrrr}
                \toprule
                Game & No optimizations & All optimizations \\
                \midrule
                \texttt{connect4.kif}                      & 32 ms   & 2042 ms \\
                \texttt{montyHall.kif}                     & 6 ms    & 53 ms \\
                \texttt{ticTacToe.kif}                     & 11 ms   & 240 ms \\
                \bottomrule
            \end{tabular}
            \caption{Analysis time of GDL games with 60 seconds time limit.}
        \end{table}
        \begin{table}
            \centering
            \scriptsize
            \begin{tabular}{lrrrr}
                \toprule
                Game & No optimizations & All optimizations \\
                \midrule
                \texttt{15puzzle.rbg}                       & 26 ms   & 1309 ms \\
                \texttt{alquerque\_lud.rbg}                 & 32 ms   & 552 ms \\
                \texttt{alquerque.rbg}                      & 33 ms   & 833 ms \\
                \texttt{amazons.rbg}                        & 25 ms   & 181 ms \\
                \texttt{amazons\_split2a.rbg}               & 42 ms   & 653 ms \\
                \texttt{amazons\_split2.rbg}                & 26 ms   & 174 ms \\
                \texttt{amazons\_split3.rbg}                & 26 ms   & 167 ms \\
                \texttt{amazons\_split5plus.rbg}            & 39 ms   & 340 ms \\
                \texttt{amazons\_split5.rbg}                & 31 ms   & 222 ms \\
                \texttt{arimaa\_fixedPosition.rbg}          & 1738 ms & timeout \\
                \texttt{arimaa.rbg}                         & 1915 ms & timeout \\
                \texttt{arimaa\_split.rbg}                  & 4193 ms & timeout \\
                \texttt{breakthrough\_10x10.rbg}            & 16 ms   & 48 ms \\
                \texttt{breakthrough\_11x11.rbg}            & 18 ms   & 51 ms \\
                \texttt{breakthrough\_12x12.rbg}            & 24 ms   & 62 ms \\
                \texttt{breakthrough\_5x5.rbg}              & 8 ms    & 28 ms \\
                \texttt{breakthrough\_6x6.rbg}              & 8 ms    & 29 ms \\
                \texttt{breakthrough\_7x7.rbg}              & 8 ms    & 32 ms \\
                \texttt{breakthrough\_9x9.rbg}              & 14 ms   & 39 ms \\
                \texttt{breakthrough.rbg}                   & 10 ms   & 37 ms \\
                \texttt{breakthrough\_split.rbg}            & 10 ms   & 43 ms \\
                \texttt{breakthru.rbg}                      & 82 ms   & 541 ms \\
                \texttt{breakthru\_split.rbg}               & 83 ms   & 467 ms \\
                \texttt{canadianDraughts.rbg}               & 99 ms   & 3587 ms \\
                \texttt{chess\_200.rbg}                     & 99 ms   & 2503 ms \\
                \texttt{chessCylinder\_kingCapture.rbg}     & 45 ms   & 1332 ms \\
                \texttt{chessCylinder.rbg}                  & 85 ms   & 3000 ms \\
                \texttt{chessGardner5x5\_kingCapture.rbg}   & 34 ms   & 1306 ms \\
                \texttt{chess\_kingCapture\_200.rbg}        & 45 ms   & 1045 ms \\
                \texttt{chess\_kingCapture.rbg}             & 45 ms   & 1370 ms \\
                \texttt{chessLosAlamos6x6\_kingCapture.rbg} & 25 ms   & 562 ms \\
                \texttt{chessQuick5x6\_kingCapture.rbg}     & 28 ms   & 856 ms \\
                \texttt{chess.rbg}                          & 88 ms   & 3045 ms \\
                \texttt{chessSilverman4x5\_kingCapture.rbg} & 19 ms   & 419 ms \\
                \texttt{chineseCheckers6.rbg}               & 524 ms  & 1944 ms \\
                \texttt{connect4.rbg}                       & 12 ms   & 54 ms \\
                \texttt{connect6.rbg}                       & 137 ms  & 717 ms \\
                \texttt{connect6\_split.rbg}                & 136 ms  & 350 ms \\
                \bottomrule
            \end{tabular}
            \caption{Analysis time of all RBG games with 60 seconds time limit (part I).}
        \end{table}

        \begin{table}
            \centering
            \scriptsize
            \begin{tabular}{lrrrr}
                \toprule
                Game & No optimizations & All optimizations \\
                \midrule
                \texttt{dashGuti.rbg}                       & 32 ms   & 843 ms \\
                \texttt{doubleChess.rbg}                    & 180 ms  & 3728 ms \\
                \texttt{englishDraughts.rbg}                & 31 ms   & 584 ms \\
                \texttt{englishDraughts\_split.rbg}         & 33 ms   & 763 ms \\
                \texttt{foxAndHounds-10x10.rbg}             & 20 ms   & 37 ms \\
                \texttt{foxAndHounds-12x12.rbg}             & 41 ms   & 61 ms \\
                \texttt{foxAndHounds.rbg}                   & 10 ms   & 26 ms \\
                \texttt{gess.rbg}                           & 3774 ms & timeout \\
                \texttt{go\_constsum.rbg}                   & 326 ms  & 1741 ms \\
                \texttt{golSkuish.rbg}                      & 37 ms   & 807 ms \\
                \texttt{gomoku\_freeStyle.rbg}              & 53 ms   & 143 ms \\
                \texttt{gomoku\_standard\_11x11.rbg}        & 24 ms   & 89 ms \\
                \texttt{gomoku\_standard\_13x13.rbg}        & 37 ms   & 114 ms \\
                \texttt{gomoku\_standard.rbg}               & 56 ms   & 154 ms \\
                \texttt{go\_nopass.rbg}                     & 352 ms  & 1788 ms \\
                \texttt{go.rbg}                             & 168 ms  & 647 ms \\
                \texttt{hex\_10x10.rbg}                     & 17 ms   & 41 ms \\
                \texttt{hex\_5x5.rbg}                       & 9 ms    & 28 ms \\
                \texttt{hex\_6x6.rbg}                       & 8 ms    & 28 ms \\
                \texttt{hex\_7x7.rbg}                       & 9 ms    & 31 ms \\
                \texttt{hex\_8x8.rbg}                       & 10 ms   & 33 ms \\
                \texttt{hex\_9x9.rbg}                       & 13 ms   & 36 ms \\
                \texttt{hex.rbg}                            & 20 ms   & 48 ms \\
                \texttt{internationalDraughts.rbg}          & 78 ms   & 3186 ms \\
                \texttt{knightthrough.rbg}                  & 12 ms   & 34 ms \\
                \texttt{knightthrough\_split.rbg}           & 13 ms   & 40 ms \\
                \texttt{lauKataKati.rbg}                    & 30 ms   & 863 ms \\
                \texttt{paperSoccer.rbg}                    & 1429 ms & 4761 ms \\
                \texttt{pentago.rbg}                        & 126 ms  & 13859 ms \\
                \texttt{pentago\_split.rbg}                 & 123 ms  & 15024 ms \\
                \texttt{pretwa.rbg}                         & 31 ms   & 833 ms \\
                \texttt{reversi\_10x10.rbg}                 & 73 ms   & 33622 ms \\
                \texttt{reversi\_4x4.rbg}                   & 44 ms   & 30935 ms \\
                \texttt{reversi\_6x6.rbg}                   & 49 ms   & 32987 ms \\
                \texttt{reversi.rbg}                        & 58 ms   & 31946 ms \\
                \texttt{skirmish.rbg}                       & 51 ms   & 1605 ms \\
                \texttt{surakarta.rbg}                      & 36 ms   & 518 ms \\
                \texttt{theMillGame\_lud.rbg}               & 1275 ms & timeout \\
                \texttt{theMillGame.rbg}                    & 40 ms   & 1921 ms \\
                \texttt{theMillGame\_split.rbg}             & 38 ms   & 1709 ms \\
                \texttt{ticTacToe.rbg}                      & 8 ms    & 30 ms \\
                \texttt{yavalath.rbg}                       & 19 ms   & 62 ms \\
                \bottomrule
            \end{tabular}
            \caption{Analysis time of all RBG games with 60 seconds time limit (part II).}
        \end{table}

\end{document}